\documentclass{article}

\usepackage[numbers]{natbib}
\usepackage{authblk}
\usepackage[title,toc,titletoc,page]{appendix}
\usepackage{algorithm}
\usepackage{capt-of}
\usepackage[noend]{algpseudocode}
\usepackage{letltxmacro}
\LetLtxMacro{\oldalgorithmic}{\algorithmic}
\LetLtxMacro{\endoldalgorithmic}{\endalgorithmic}
\renewenvironment{algorithmic}[1][0]{%
  \hrulefill\par
  \oldalgorithmic[#1]}
  {\endoldalgorithmic\par
   \vspace*{-.5\baselineskip}
   \hrulefill\par
  }

\usepackage{fullpage}
\usepackage{url}            
\usepackage{booktabs}       
\usepackage{amsfonts}       
\usepackage{dsfont}
\usepackage{nicefrac}       
\usepackage{microtype}      
\usepackage{amsthm}
\usepackage{float}
\usepackage[english]{babel}
\usepackage{mathtools}
\usepackage{verbatim}
\usepackage{tabu}
\usepackage{amsmath}
\usepackage{mathrsfs}
\usepackage{enumitem}
\usepackage{amsfonts}
\usepackage{graphicx}
\usepackage{parskip}
\usepackage{color}
\usepackage{amssymb}
\usepackage{epstopdf}

\newtheorem{theorem}{Theorem}
\newtheorem{lemma}{Lemma}
\newtheorem{remark}[lemma]{Remark}

\newtheorem{proposition}{Proposition}
\newtheorem{fact}{Fact}
\newtheorem{corollary}{Corollary}
\newtheorem{definition}[lemma]{Definition}
\renewenvironment{proof}[1][Proof]{\begin{trivlist}
\item[\hskip \labelsep {\bfseries #1}]}{\qed\end{trivlist}}

\newcommand*\R[0]{\mathbb{R}}

\newcommand*\ddt[0]{\frac{d}{d t}}

\newcommand*\lin[1]{\left\langle #1\right\rangle}
\newcommand*\E[1]{\mathbb{E}\left[#1\right]}
\newcommand*\Ep[2]{\mathbb{E}_{#1}\left[#2\right]}

\newcommand*\lrbb[1]{\left\{#1\right\}}
\newcommand*\lrp[1]{\left(#1\right)}
\newcommand*\lrn[1]{\left\|#1\right\|}

\newcommand*\A[0]{\mathcal{A}}

\newcommand*\ind[1]{{\mathds{1}\lrbb{#1}}}

\renewcommand*{\qed}{\hfill\ensuremath{\blacksquare}}

\newcommand*{\p}{p}

\newcommand*{\pmu}{\boldsymbol{\mu}}

\def\vx{{\bf x}}
\def\vX{{\bf X}}
\def\vy{{\bf y}}
\def\vz{{\bf z}}
\def\vv{{\bf v}}
\def\rd{\mathrm{d}}
\def\rT{{\rm T}}
\def\ball{\mathbb{B}}
\def\mI{\ensuremath{\mathbb{I}}}

\newcommand{\mixingerror}{\epsilon}
\newcommand{\mixingtime}{\tau}

\newcommand{\density}{p}
\newcommand{\initial}{p^0}
\newcommand{\target}{\ensuremath{p^*}}

\newcommand{\stationary}{\ensuremath{\Pi}}
\newcommand{\warmparam}{\beta}
\newcommand{\func}{U}

\newcommand{\proposal}{\mathcal{P}}

\newcommand{\step}{h}

\newcommand{\stepbound}{\tilde{w}}
\newcommand{\scalar}{\alpha}
\newcommand{\tvscalar}{\gamma}

\newcommand{\transition}{\mathcal{T}}

\newcommand{\tagmala}{\text{\tiny MALA}}

\newcommand{\flow}{\phi}
\newcommand{\conductance}{\Phi}

\newcommand{\radius}{r}
\newcommand{\truncball}{\mathcal{R}}
\newcommand{\maxgrad}{\mathcal{D}}

\newcommand{\convexset}{\ensuremath{\mathcal{K}}}

\newcommand{\scparam}{m}
\newcommand{\smoothness}{L}
\newcommand{\condition}{\ensuremath{\kappa}}

\newcommand{\dims}{d}

\newcommand{\logsobconstU}{\rho_U}

\newcommand{\real}{\ensuremath{\mathbb{R}}}
\newcommand{\Ind}{\ensuremath{\mathbb{I}}}
\newcommand{\Exs}{\ensuremath{{\mathbb{E}}}}
\newcommand{\Prob}{\ensuremath{{\mathbb{P}}}}

\DeclarePairedDelimiterX{\infdivx}[2]{(}{)}{%
  #1\;\delimsize\|\;#2%
}
\newcommand{\kldiv}{\text{KL}\infdivx}

\newcommand{\NORMAL}{\ensuremath{\mathcal{N}}}

\newcommand{\brackets}[1]{\left[ #1 \right]}
\newcommand{\parenth}[1]{\left( #1 \right)}

\newcommand{\braces}[1]{\left\{ #1 \right \}}

\newcommand{\angles}[1]{\left\langle #1 \right \rangle}
\newcommand{\tp}{^\top}

\newcommand{\matsnorm}[2]{|\!|\!| #1 | \! | \!|_{{#2}}}
\newcommand{\vecnorm}[2]{\left\| #1\right\|_{#2}}

\title{Sampling Can Be Faster Than Optimization}

\author[a]{Yi-An Ma\thanks{yianma@berkeley.edu}}
\author[b]{Yuansi Chen\thanks{yuansi.chen@berkeley.edu}}
\author[a]{Chi Jin\thanks{chijin@cs.berkeley.edu}}
\author[a]{Nicolas Flammarion\thanks{flammarion@berkeley.edu}}
\author[a, b]{Michael I. Jordan\thanks{jordan@cs.berkeley.edu}}

\affil[a]{Department of Electrical Engineering and Computer Sciences, University of California, Berkeley, CA 94720}
\affil[b]{Department of Statistics, University of California, Berkeley, CA 94720}

\begin{document}
\maketitle

\begin{abstract}
Optimization algorithms and Monte Carlo sampling algorithms have provided
the computational foundations for the rapid growth in applications of statistical
machine learning in recent years.  There is, however, limited theoretical
understanding of the relationships between these two kinds of methodology,
and limited understanding of relative strengths and weaknesses.  Moreover,
existing results have been obtained primarily in the setting of convex
functions (for optimization) and log-concave functions (for sampling).
In this setting, where local properties determine global properties,
optimization algorithms are unsurprisingly more efficient computationally
than sampling algorithms.  We instead examine a class of \emph{nonconvex}
objective functions that arise in mixture modeling and multi-stable systems.
In this nonconvex setting, we find that the computational complexity of
sampling algorithms scales linearly with the model dimension while that
of optimization algorithms scales exponentially.
\end{abstract}

{M}achine learning and data science are fields that blend computer science
and statistics so as to solve inferential problems whose scale and complexity require
modern computational infrastructure.  The algorithmic foundations on which
these blends have been based repose on two general computational strategies, both which
have their roots in mathematics---optimization and Markov chain Monte Carlo
(MCMC) sampling.  Research on these strategies has mostly proceeded separately,
with research on optimization focused on estimation and prediction problems,
and with research on sampling focused on tasks that require uncertainty
estimates, such as forming credible intervals and conducting hypothesis
tests.  There is a trend, however, towards the use of common methodological
elements within the two strands of research~\citep{Amit_Grenander,Amit,Gareth_Gauss_Siedel,Meng_EM,
Dalalyan_JRSSB,Moulines_ULA,Dalalyan_user_friendly,Xiang_underdamped,Xiang_overdamped,
dwivedi2018log, Mangoubi1,Mangoubi2}.  In particular, both strands
have focused on the use of gradients and stochastic gradients---rather than
function values or higher-order derivatives---as providing a useful compromise
between the computational complexity of individual algorithmic steps and the
overall rate of convergence.  Empirically, the effectiveness of this compromise
is striking.  But the relative paucity of theoretical research linking
optimization and sampling has limited the flow of ideas; in particular,
the rapid recent advance of theory for optimization~\citep[see, e.g.,][]{Nesterov_intro}
has not yet translated into a similarly rapid advance of the theory for sampling.
Accordingly, machine learning has remained limited in its inferential scope,
with little concern for estimates of uncertainty.

Theoretical linkages have begun to appear in recent work~\citep[see, e.g.,][]
{Dalalyan_JRSSB,Moulines_ULA, Dalalyan_user_friendly,Xiang_underdamped,Xiang_overdamped,dwivedi2018log, Mangoubi1,Mangoubi2},
where tools from
optimization theory have been used to establish rates of convergence---notably
including non-asymptotic dimension dependence---for MCMC sampling.  The overall message from
these results is that sampling is slower than optimization---a message which
accords with the folk wisdom that sampling approaches are warranted only if
there is need for the stronger inferential outputs that they provide.  These
results are, however, obtained in the setting of \emph{convex} functions.
For convex functions, global properties can be assessed via local information.
Not surprisingly, gradient-based optimization is well suited to such a setting.

Our focus is the \emph{nonconvex} setting.  We consider a broad class of
problems that are strongly convex outside of a bounded region, but nonconvex
inside of it.  Such problems arise, for example, in Bayesian mixture model problems~\citep{GMM_book,Bayesian_GMM}, and in the noisy
multi-stable models that are common in statistical physics~\citep{Kramers,Landau}.
We find that when the nonconvex region has a constant and nonzero radius
in $\R^d$, the MCMC methods converge to $\epsilon$ accuracy in
$\widetilde{\mathcal{O}}\left({d}/{\epsilon}\right)$ or
$\widetilde{\mathcal{O}}\left(d^2 \ln\left(1/\epsilon\right)\right)$
steps whereas any optimization approach converges in
$\widetilde{\Omega} \left( ({1}/{\epsilon})^d \right)$ steps.
Thus, for this class of problems, sampling is more effective than
optimization.

We obtain these polynomial convergence results for the MCMC algorithms in
the nonconvex setting by working in continuous time and separating the problem
into two subproblems: given the target distribution we first exploit the properties
of a weighted Sobolev space endowed with that target distribution to obtain convergence
rates for the continuous dynamics, and we then discretize and find the appropriate
step size to retain those rates for the discretized algorithm.  This general framework
allows us to strengthen recent results in the MCMC literature~\citep{Eberle_ULA,Eberle_HMC,
Xiang_Nonconvex, Local_Nonconvex_W2} and examine a broader class of algorithms including
the celebrated Metropolis-Hastings method.

\section{Polynomial Convergence of MCMC Algorithms}
\renewcommand{\figurename}{Algorithm}

\begin{figure}[t!]
\begin{center}
{(Metropolis Adjusted) Langevin Algorithm}
\end{center}
\begin{algorithmic}
\State{Input: $\vx^{0}$, stepsizes $\{h^{k}\}$}
\For{$k = 0, 1, 2, \ldots, K-1$}
\State{
$\vx^{k+1} \leftarrow \vx^{k} - h^{k} \nabla U(\vx^{k}) + \xi$
}
\If{$\dfrac{\density\parenth{\vx^{k}|\vx^{k+1}} \target(\vx^{k})}{\density\parenth{\vx^{k+1}|\vx^{k}} \target\parenth{\vx^{k+1}}} < u$}
\State{$\vx^{k+1} \leftarrow \vx^{k}$}
\Comment{Metropolis Adjustment}
\EndIf
\EndFor
\State{Return $\vx^{K}$}
\end{algorithmic}
\captionof{figure}{
The (Metropolis adjusted) Langevin algorithm is a gradient-based MCMC algorithm.
In each step, one simulates $\xi\sim\mathcal{N}(0, 2h^{k} \mI)$ and
$u\sim\mathcal{U}[0,1]$ a uniform random variable between $0$ and $1$.
The conditional distribution $\density\parenth{\vx^{k}|\vx^{k+1}}$ is
the normal distribution centered at $\vx^{k} - h^{k} \nabla U(\vx^{k})$
and $p^*$ is the target distribution.  Without the Metropolis adjustment
step, the algorithm is called the unadjusted Langevin algorithm (ULA).
Otherwise, it is called the Metropolis Adjusted Langevin Algorithm (MALA).
}
\label{alg:mala}
\end{figure}

The \emph{Langevin algorithm} is a family of gradient-based MCMC sampling
algorithms~\citep{Langevin_origin,MALA,durmus2017}.  We present pseudocode for two variants of the algorithm
in Algorithm~\ref{alg:mala}, and, by way of comparison, we provide pseudocode for classical
gradient descent (GD) in Algorithm~\ref{alg:gd}.  The variant of the Langevin algorithm
which does not include the ``if'' statement is referred to as the \emph{unadjusted
Langevin algorithm} (ULA); as can be seen, it is essentially the same as
GD, differing only in its incorporation of a random term
$\xi\sim\mathcal{N}(0, 2h^{k} \mI)$ in the update.  The variant that
includes the ``if'' statement is referred to as the \emph{Metropolis Adjusted
Langevin Algorithm} (MALA); it is the standard Metropolis-Hastings algorithm
applied to the Langevin setting.
It is worth noting that ULA differs from stochastic optimization algorithms
in the scaling of the variance of the random term $\xi$: In stochastic gradient descent,
the variance of $\xi$ scales as squared stepsize, $\left(h^{k}\right)^2$.

We consider sampling from a smooth target distribution $\p^*$ that is strongly
log-concave outside of a region.  That is, for $\p^* \propto e^{- U}$, we assume
that $U$ is $m$-strongly convex outside of a region of radius $R$ and is $L$-Lipschitz
smooth.\footnote{$U$ being $L$-Lipschitz smooth means that
$\nabla U$ is $L$-Lipschitz continuous. Smoothness is crucial for the convergence
of gradient-based methods~\citep{roberts_counter}.} (See 
Supplement A for a formal statement of the assumptions). Let $\kappa=L/m$ denote
the \emph{condition number} of $U$; this is a parameter which
measures how much $U$ deviates from an isotropic quadratic function outside of
the region of radius $R$. We prove convergence of the Langevin sampling
algorithms for this target, establishing a convergence rate.  Given an error tolerance
$\mixingerror \in (0, 1)$ and an initial distribution $p^0$, define the
$\mixingerror$-mixing time in total variation distance as
\begin{align*}
   \mixingtime(\mixingerror; \initial) = \min \braces{k | \vecnorm{\density^k - \target}{\text{TV}} \leq \mixingerror}.
 \end{align*}

\begin{theorem}\label{thm:MCMC_main}
 Consider Algorithm~\ref{alg:mala} with initialization
 $\initial=\NORMAL\parenth{0,\frac{1}{\smoothness}\Ind_\dims}$ and error tolerance
 $\mixingerror \in (0, 1)$.  Then ULA with step size $h^k=\mathcal{O}\left(e^{-16LR^2} \kappa^{-1} L^{-1} \epsilon^2/d\right)$ satisfies
  \begin{align}
     \mixingtime_{\text{ULA}}(\mixingerror, \initial) \leq \mathcal{O}\parenth{ e^{32\smoothness R^2} \condition^2 \dfrac{\dims}{\mixingerror^2} \ln\left(\dfrac{\dims}{\mixingerror^2}\right) }. \label{thm:ULA_convergence_main}
  \end{align}
  For MALA with step size $h^k=\mathcal{O}\left(e^{-8LR^2} \kappa^{-1/2} L^{-1}\left(d\ln\kappa+\ln 1/\epsilon\right)^{-1/2} d^{-1/2}\right)$,
    \begin{align}
         \mixingtime_{\text{MALA}}(\mixingerror, \initial) \leq \mathcal{O}\parenth{ \frac{e^{40\smoothness R^2}}{\scparam}\condition^{3/2}\dims^{1/2} \parenth{\dims \ln\condition + \ln\parenth{\frac{1}{\mixingerror}}}^{3/2} }. \label{thm:MALA_convergence_main}
     \end{align}
\end{theorem}

\begin{figure}[t!]
\begin{center}
{Gradient Descent}
\end{center}
\begin{algorithmic}
\State{Input: $\vx^{0}$, stepsizes $\{h^{k}\}$}
\For{$k = 0, 1, 2, \ldots, K-1$}
\State{$\vx^{k+1} \leftarrow \vx^{k} - h^{k} \nabla U(\vx^{k})$}
\EndFor
\State{Return $\vx^{K}$}
\end{algorithmic}
\captionof{figure}{Gradient descent (GD) is a classical gradient-based optimization
algorithm which updates $\vx$ along the negative gradient direction.}
\label{alg:gd}
\end{figure}

\renewcommand{\figurename}{Figure}
\setcounter{figure}{0}

Comparing \eqref{thm:ULA_convergence_main} with \eqref{thm:MALA_convergence_main}, we see
that the Metropolis adjustment improves the mixing time of ULA to a logarithmic
dependence in $\mixingerror$, while sacrificing a factor of dimension $d$.
(Note, however, that these are upper bounds, and they depend on our specific
setting and our assumptions.
It should not be inferred from our results that ULA is generically faster than MALA in terms of dimension dependence.)
Comparing \eqref{thm:ULA_convergence_main} and \eqref{thm:MALA_convergence_main} with
previous results in the literature that provide upper bounds on the mixing time of
ULA and MALA for strongly convex potentials $U$~\citep{Dalalyan_JRSSB,
Moulines_ULA,Dalalyan_user_friendly,Xiang_underdamped,Xiang_overdamped,
dwivedi2018log,Mangoubi1,Mangoubi2}, we find that the local nonconvexity
results in an extra factor of $e^{\mathcal{O}\left( \smoothness R^2 \right)}$.
Thus, when the Lipschitz smoothness $L$ and radius of the nonconvex region
$R$ satisfy $LR^2$ is $\mathcal{O}\left(\log d\right)$, the computational
complexity is polynomial in dimension $\dims$.

Our proof of Theorem~\ref{thm:MCMC_main} involves a two-step framework that
applies more widely than our specific setting.  We first use properties
of $p^*\propto e^{-U}$ to establish linear convergence of a continuous
stochastic process that underlies Algorithm~\ref{alg:mala}.
We then discretize, finding an appropriate step size for the algorithm
to converge to the desired accuracy.  These two parts can be tackled
independently.  In this section, we provide an overview of the first part
of the argument in the case of the MALA algorithm.  The details, as well
as a presentation of the second part of the argument, are provided in
Supplement B.

Letting $t = \sum_{i=1}^k h^{i}$, assumed finite, a standard limiting process yields
the following stochastic differential equation (SDE) as a continuous-time limit of
Algorithm~\ref{alg:mala}: $\rd \vX_t = - \nabla U(\vX_t)\rd t + \sqrt{2} \rd B_t$,
where $B_t$ is a Brownian motion.  To assess the rate of convergence of this
SDE, we make use of the Kullback-Leibler (KL) divergence, which upper bounds the
total variation distance and allows us to obtain strong convergence guarantees
that include dimension dependence.  Denoting the probability distribution of
$\vX_t$  as $\tilde{\p}_t$, we obtain (see the derivation in Supplement B.2) the following
time derivative of the divergence of $\tilde{\p}_t$ to the target distribution $p^*$:
\begin{align}
\dfrac{\rd}{\rd t} \kldiv{\tilde{\p}_t}{p^*}
= - \Ep{\tilde{\p}_t}{\lrn{ \nabla \ln\lrp{\frac{\tilde{\p}_t(\vx)}{\p^*(\vx)}} }^2}. \label{eq:kl_derivative}
\end{align}

The property of $\p^* \propto e^{- U}$ that we require to turn this time derivative
into a convergence rate is that it satisfies a \emph{log-Sobolev inequality}.
Considering the Sobolev space defined by the weighted $L^2$ norm:
$\int g(\vx)^2 p^*(\vx) \rd \vx$, we say that $\p^*$ satisfies a log-Sobolev
inequality if there exists a constant $\rho>0$ such that for any smooth
function $g$ on $\R^d$, satisfying $\int_{\R^d} g(\vx) \p^*(\vx) \rd \vx = 1$,
we have:
\begin{align*}
\int g(\vx) \ln g(\vx) \cdot \p^*(\vx) \rd \vx
\leq \dfrac{1}{2\rho} \int \dfrac{ \lrn{ \nabla g(\vx) }^2}{g(\vx)} \p^*(\vx) \rd \vx.
\end{align*}
The largest $\rho$ for which this inequality holds is said to be the
\emph{log-Sobolev constant} for the objective $U$.  We denote it as $\rho_U$.
Taking $g = \tilde{\p}_t / \p^*$, we obtain:
\begin{align}
\kldiv{\tilde{\p}_t}{p^*}
&= \Ep{\tilde{\p}_t}{\ln\lrp{\frac{\tilde{\p}_t(\vx)}{\p^*(\vx)}}} \nonumber\\
&\leq \dfrac{1}{2\rho_U} \Ep{\tilde{\p}_t}{\lrn{ \nabla \ln\lrp{\frac{\tilde{\p}_t(\vx)}{\p^*(\vx)}} }^2 }. \label{eq:log-Sobolev_main}
\end{align}
Note the resemblance of this bound to the Polyak-\L{}ojasiewicz condition~\citep{Polyak}
used in optimization theory for studying the convergence of smooth and strongly
convex objective functions---in both cases the difference from the current iterate
to the optimum is upper bounded by the norm of the gradient squared.
Combining \eqref{eq:kl_derivative} with \eqref{eq:log-Sobolev_main},
we derive the promised linear convergence rate for the continuous process:
\begin{align*}
\dfrac{\rd}{\rd t} \kldiv{\tilde{\p}_t}{p^*}
\leq - 2\rho_U \kldiv{\tilde{\p}_t}{p^*}.
\end{align*}
In
Supplement B.2, we present similar results for the ULA algorithm, again using
the KL divergence.

The next step is to bound $\rho_U$ in terms of the basic smoothness and local
nonconvexity assumptions in our problem.  We first require an approximation result:
\begin{lemma}
For $U$ $m$-strongly convex outside of a region of radius $R$ and $L$-Lipschitz smooth,
there exists $\hat{U}\in C^1(\R^d)$ such that $\hat{U}$ is $m/2$ strongly convex
on $\R^d$, and has a Hessian that exists everywhere on $\R^d$.  Moreover, we have
$\sup\left(\hat{U}(\vx)-U(\vx)\right) - \inf\left(\hat{U}(\vx)-U(\vx)\right) \leq 16 LR^2$.
\label{lemma:hat_U}
\end{lemma}
The proof of this lemma is presented in Supplement B.1.  The existence of the
smooth approximation established in this lemma can now be used to bound the
log-Sobolev constant using standard results.
\begin{proposition}
  \label{thm:log_sobolev_constant}
  For $\p^* \propto e^{- U}$, where $U$ is $m$-strongly convex outside of a
  region of radius $R$ and $L$-Lipschitz smooth,
  \begin{align}
    \rho_U \geq \dfrac{m}{2} e^{-16LR^2}.
  \end{align}
\end{proposition}
\begin{proof}
For $m/2$-strongly convex $\hat{U}\in C^1(\R^d)$ whose Hessian $\nabla^2 \hat{U}(\vx)$
exists everywhere on $\R^d$, the distribution $e^{- \hat{U}(\vx)}$ satisfies the
Bakry-Emery criterion \citep{bakryEmery} for a strongly log-concave density, which
yields:
\begin{align}
\rho_{\hat{U}} \geq \dfrac{m}{2}.
\end{align}
We use the Holley-Stroock theorem \citep{HolleyStroock} to obtain:
\begin{align}
\rho_U \geq
\dfrac{m}{2} e^{- \left|\sup\left(\hat{U}(\vx)-U(\vx)\right) - \inf\left(\hat{U}(\vx)-U(\vx)\right)\right|}
\geq \dfrac{m}{2} e^{-16LR^2}.
\end{align}

\end{proof}

We see from this proof outline that our approach enables one to adapt existing
literature on the convergence of diffusion processes~\citep{Ledoux_log_Sobolev_diffusion,
Villani_optimal_transport,Sampling_as_optimization} to work out suitable
log-Sobolev bounds and thereby obtain sharp convergence rates in terms
of distance measures such as the KL divergence and total variation.
This contributes to the existing literature on convergence of MCMC~\citep{Kannan,Rosenthal_JASA,rosenthal2002,roberts2001,roberts2016} by providing
non-asymptotic guarantees on computational complexity.
The detailed proof also reveals that the log-Sobolev constant $\rho_U$ is
largely determined by the global qualities of $U$ where most of the probability
mass is concentrated; local properties of $U$ have limited influence on $\rho_U$.
Since this is a property of the Sobolev space defined by the
$p^*$-weighted $L^2$ norm, the favorable convergence rates of the Langevin algorithms
can be expected to generalize to other sampling algorithms~\citep[see, e.g.,][]{Acceleration_MCMC}.

\section{Exponential Dependence on Dimension for Optimization}

It is well known that finding global minima of a general nonconvex optimization
problem is NP-hard \citep{PrateekBook}.  Here we demonstrate that it is also hard
to find an $\epsilon$ approximation to the optimum of a Lipschitz-smooth,
locally nonconvex objective function $U$, for any algorithm in a general class of optimization algorithms.

Specifically, we consider a general iterative algorithm family $\mathcal{A}$
which, at every step $k$, is allowed to query not only the function value of
$U$ but also its derivatives up to any fixed order at a chosen point $\vx^k$.
Thus the algorithm has access to the vector ($\{ U(\vx^k), \nabla U(\vx^k),
\cdots, \nabla^n U(\vx^k) \}, {\text{for any fixed }} n\in\mathcal{N}$).
Moreover, the algorithm can use the entire query history to determine the
next point $\vx^{k+1}$, and it can do so randomly or deterministically.
In the following theorem, we prove that the number of iterations for any
algorithm in $\mathcal{A}$ to approximate the minimum of $U$ is necessarily
exponential in the dimension $d$.

\begin{theorem}[Lower bound for optimization] \label{thm:lower_opt_main}
For any $R>0$, $L \geq 2m > 0$, and
$\epsilon \le \mathcal{O}(LR^2)$,
there exists an objective function, $U:\R^d \rightarrow \R$, which is
$m$-strongly convex outside of a region of radius $R$ and $L$-Lipschitz smooth,
such that any algorithm in $\mathcal{A}$ requires at least
$K = \Omega( ( LR^2/\epsilon )^{d/2})$ iterations to guarantee that
$\min_{k \le K}|U(\vx^K) - U(\vx^*)| < \epsilon$ with constant probability.
\label{theorem:lower_bound}
\end{theorem}

\begin{figure}
\centering
\includegraphics[scale=0.5]{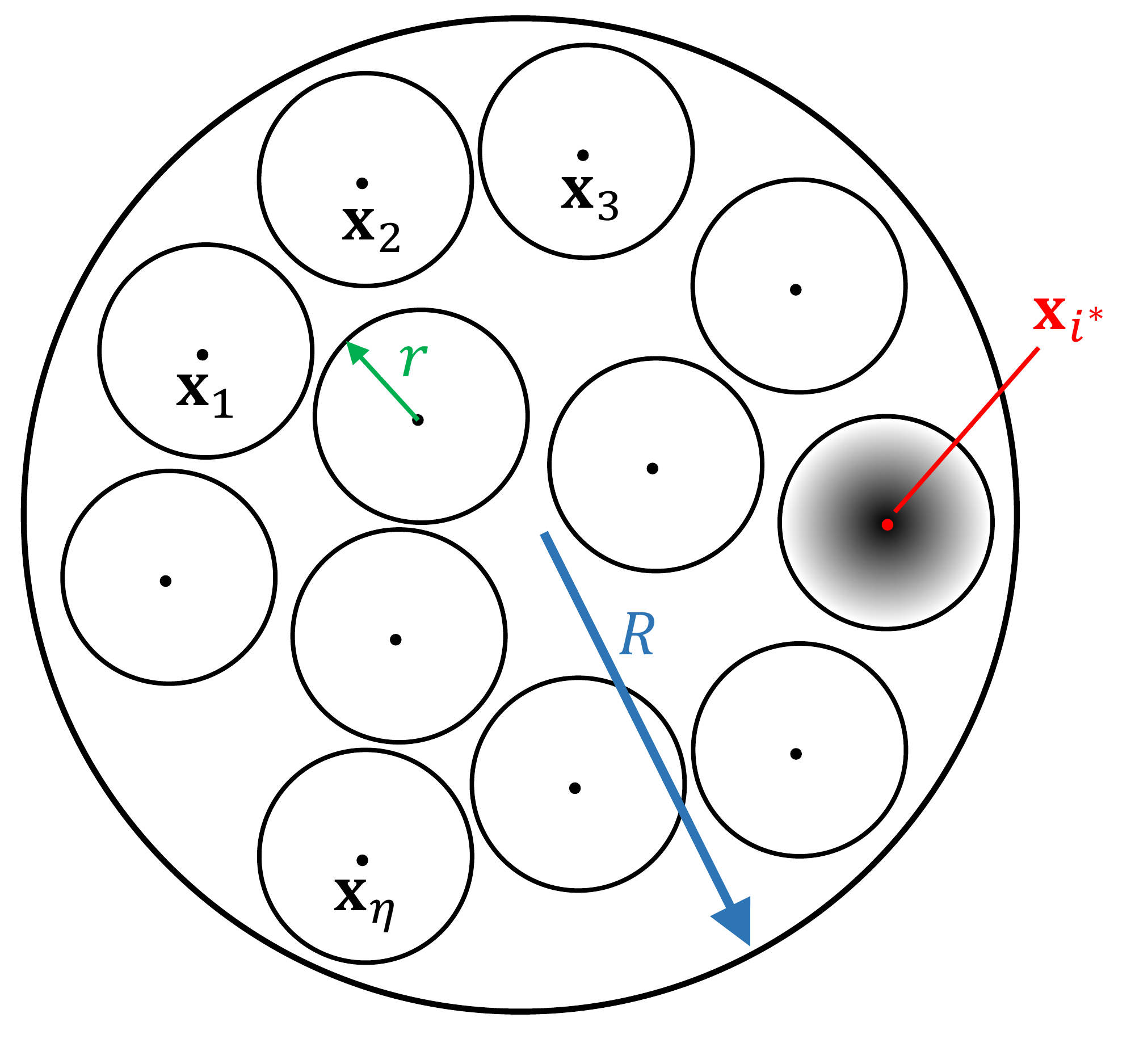}
\caption{Depiction of an instance of $U(\vx)$, inside the region of
radius $R$. that attains the lower bound.}
\label{fig:cartoon_1}
\end{figure}
We remark that Theorem~\ref{thm:lower_opt_main} is an information-theoretic
result based on the class of iterative algorithms $\A$ and the forms of
the queries to this class.  It is thus an unconditional statement that
does not depend on conjectures such as $P\neq NP$ in complexity theory.
We also note that if the goal is only to find stationary
points instead of the optimum, then the problem becomes easier, requiring
only $(1/\epsilon)^2$ gradient queries to converge~\citep{Duchi_stationary}.

A depiction of an example that achieves this computational lower bound
is provided in Fig.~\ref{fig:cartoon_1}.  The idea is that we can pack
exponentially many balls of radius less than $R/3$ inside a region of radius $R$.
We can arbitrarily assign the minimum $\vx^*$ to one of the balls, assigning
a larger constant value to the other balls.  We show that the number of
queries needed to find the specific ball containing the minimum is exponential
in $d$.  Moreover, the difference from $U(\vx^*)$ to any other point outside
of the ball is $\mathcal{O}(LR^2)$, which can be significant.

This example suggests that the lower-bound scenario will be realized in
cases in which regions of attraction are small around a global minimum and
behavior within each region of attraction is relatively autonomous.
This phenomenon is not uncommon in multi-stable physical systems.
Indeed, in non-equilibrium statistical physics, there are examples
where the global behavior of a system can be treated approximately as
a set of local behaviors within stable regimes plus Markov transitions
among stable regimes~\citep{Qian_chaos}.  In such cases, when the regions of
attraction are small, the computational complexity to find the global
minimum can be combinatorial.  In Sec.~\ref{sec:Example_GMM}, we explicitly
demonstrate that this combinatorial complexity holds for a Gaussian mixture model.

\paragraph{Why Can't One Optimize in Polynomial Time Using the Langevin Algorithm?}
Consider the rescaled density function $q^{*}_{\beta}\propto e^{- \beta U}$.
A line of research beginning with simulated annealing~\citep{kirkpatrick1983optimization}
uses a sampling algorithm to sample from $q^{*}_{\beta}$, doing so for increasing
values of $\beta$, and uses the resulting samples to approximate
$\vx^*=\arg\min_{\vx\in\mathbb{R}^d} U(\vx)$.  In particular, simply
returning one of the samples obtained for sufficiently large $\beta$
yields an output that is close to the optimum with high probability.
This suggests the following
question: Can we use the Langevin algorithm
to generate samples from $q^{*}_{\beta}$, and thereby obtain an approximation
to $\vx^*$ in a number of steps polynomial in $d$?

In the following Corollary~\ref{lemma:relation}, we demonstrate that this is
\emph{not} possible: We need $\beta=\widetilde{\Omega}\left(d/\epsilon\right)$
so that a sample $\vx$ from $q^{*}_{\beta}$ will satisfy $\lrn{\vx - \vx^*}
\leq \epsilon$ with constant probability.
(Here $\widetilde{\Omega}$ means we have omitted logarithmic factors.)
This requires the Lipschitz smoothness
of $U$ to scale with $d$, which in turn causes the sampling complexity
to scale exponentially with $d$, as established in \eqref{thm:ULA_convergence_main}
and \eqref{thm:MALA_convergence_main}.

\begin{corollary}
\label{lemma:relation}
There exists an objective function $U$ that is
$m$-strongly convex outside of a region of radius $2R$ and $L$-Lipschitz smooth,
such that, for $\hat{\vx} \sim q^{*}_{\beta}$, it is necessary that
$\beta=\widetilde{\Omega}\left(d/\epsilon\right)$ in order to have
$U(\hat{\vx})-U\left(\vx^*\right)<\epsilon$ with constant probability.
Moreover, the number of iterations required for the Langevin algorithms
to achieve $U(\vx^K)-U\left(\vx^*\right)<\epsilon$ with constant probability
is $K = e^{\widetilde{\mathcal{O}} \left( d \cdot LR^2/\epsilon \right)}$.
\end{corollary}

It should be noted that this upper bound for the Langevin algorithms agrees with
the lower bound for optimization algorithms in Theorem~\ref{theorem:lower_bound}
up to a factor of $LR^2/\epsilon$ in the exponent.
Intuitively this is because in the lower bound for optimization complexity we are considering the
most optimistic scenario for optimization algorithms, where a hypothetical algorithm can
determine whether one region of radius $\sqrt{\epsilon/L}$ (as depicted in Fig.~\ref{fig:cartoon_1})
contains the global minimum or not with only one query (of function value and $n$-th order derivatives).
When using the Langevin algorithms,
more steps are required to explore each local region to a constant level of confidence.

\section{Parameter Estimation from Gaussian Mixture Model: Sampling versus Optimization}
\label{sec:Example_GMM}

We have seen that for problems with local nonconvexity, the computational
complexity for the Langevin algorithm is polynomial in dimension whereas it
is exponential in dimension for optimization algorithms.  These are, however,
worst-case guarantees.  It is important to consider whether they also hold
for natural statistical problem classes and for specific optimization algorithms.
In this section, we study the Gaussian mixture model, comparing Langevin
sampling and the popular expectation-maximization (EM) optimization algorithm.

Consider the problem of inferring the mean parameters of a Gaussian mixture model,
$\pmu=\{\mu_1,\cdots,\mu_M\}\in\mathbb{R}^{d\times M}$, when $N$ data points
are sampled from that model. Letting $\vy=\{y_1,\cdots,y_N\}$
denote the data, we have:
\begin{align}
p\left(y_n | \pmu\right) &= \sum_{i=1}^M \dfrac{\lambda_i}{Z_i}
\exp\left( - \dfrac{1}{2} (y_n - \mu_i)^\rT \Sigma_i^{-1} (y_n - \mu_i) \right)
\nonumber\\
&+ \left( 1 - \sum_{i=1}^M \lambda_i \right) p_0(y_n),
\label{eq:GMM}
\end{align}
where $Z_i$ are normalization constants and $\sum_{i=1}^M \lambda_i \leq 1$.
$p_0(y_n)$ represents general constraints on the data (e.g., data may be distributed
inside a region or may have sub-Gaussian tail behavior).
The objective function is given by the log posterior distribution: $U(\pmu) = -\log p(\pmu) - \sum_{n=1}^N \log p\left(y_n | \pmu\right)$. Assume data are distributed in a bounded region ($\lrn{y_n}\leq R$) and take $p_0(y_n)= \ind{\|y_n\|\leq R} / Z_0$.

We prove in
Supplement D
that for a suitable choice of the prior $p(\pmu)$ and weights $\{\lambda_i\}$,\footnote{
We specify in Supplement D that the weights $\lambda_i/Z_i$ scale as the variance $\|\Sigma_i\|_2^2$ and the prior satisfies $p(\pmu) \propto \exp\left( - m \left(\|\pmu\|_F - {\sqrt{M}} R\right)^2 \ind{\|\pmu\|_F \geq {\sqrt{M}} R} \right)$.}
the objective function is
Lipschitz smooth and strongly convex for $\lrn{\pmu} \geq 2 R\sqrt{M}$.
Therefore, taking $MR^2 = \mathcal{O}(\log d)$, the ULA and MALA algorithms converge to $\epsilon$
accuracy within $K \leq \widetilde{\mathcal{O}}\left(d^3/\epsilon\right)$ and
$K \leq \widetilde{\mathcal{O}}\left(d^3 \ln^2\left(1/\epsilon\right)\right)$
steps, respectively.

The EM algorithm updates the value of $\pmu$ in two steps.
In the expectation (E) step a weight is computed for each data point and each mixture component,
using the current parameter value $\pmu_k$. In the maximization (M) step the value of $\pmu_{k+1}$
is updated as a weighted sample mean (see
Supplement D.2 for a more detailed description). It is standard to initialize the EM algorithm by
randomly selecting $M$ data points (sometimes with small perturbations) to form $\pmu_0$. We demonstrate in
Supplement D.2
that under the condition that $MR^2 = \mathcal{O}(\log d)$, there exists a dataset $\{y_1,\cdots,y_N\}$ and covariances $\{\Sigma_1,\cdots,\Sigma_M\}$, such that the EM algorithm requires more than $K\geq\min\{\mathcal{O}(d^{1/\epsilon}), \mathcal{O}(d^d)\}$ queries to converge if one initializes the algorithm close to the given data points.
That is, for large $\epsilon$, the computational complexity of the EM algorithm depends on $d$ with
arbitrarily high order (depending on $\epsilon$);
for small $\epsilon$, the computational complexity of the EM algorithm scales exponentially with $d$.
The latter case corresponds to our lower bound in Theorem~\ref{theorem:lower_bound} when taking the
radius of the convex region of $\pmu$ to scale with $\sqrt{\log d}$.
Therefore, it is significantly harder for the EM algorithm to converge if we initialize the algorithm
close to the given data points. This accords with practical implementations of EM
algorithms, where heuristic, problem-dependent methods are often employed during initialization with the
aim of decreasing the overall computation burden~\citep{Vampala_GMM}.
The same behavior appears in the gradient-based optimization methods (e.g., gradient descent), where the algorithms are trapped in one of the many local optima.

We also investigated this dichotomy experimentally.
We considered synthetic data $\{y_1,\cdots,y_N\}$ with sparse entries and let the nonzero entries follow uniform distribution on $[-1,1]$.
We used EM and ULA algorithms to infer the mean parameters $\pmu=\{\mu_1,\cdots,\mu_M\}$ in the Gaussian mixture model to obtain maximum a posteriori and mean estimates, respectively.
Accuracy of the maximum a posteriori estimate was measured in terms of the objective function value $U$,
while that of the mean estimates was measured in terms of both the expected objective function value $\E{U(\pmu)}$ (or the cross entropy between the sampled distribution and the posterior) and the expected mean parameters $\E{\pmu}$.
See Supplement E for detailed experimental settings.
In Fig.~\ref{fig:experiment}, we show the scaling of the number of gradient queries required to converge with respect to the dimension $d$ of data.
We observe that EM with random initialization from the data requires exponentially many gradient queries to converge, while ULA converges in an approximately linear number of gradient queries, corroborating our theoretical analysis.
\begin{figure}
\centering
\includegraphics[scale=0.35]{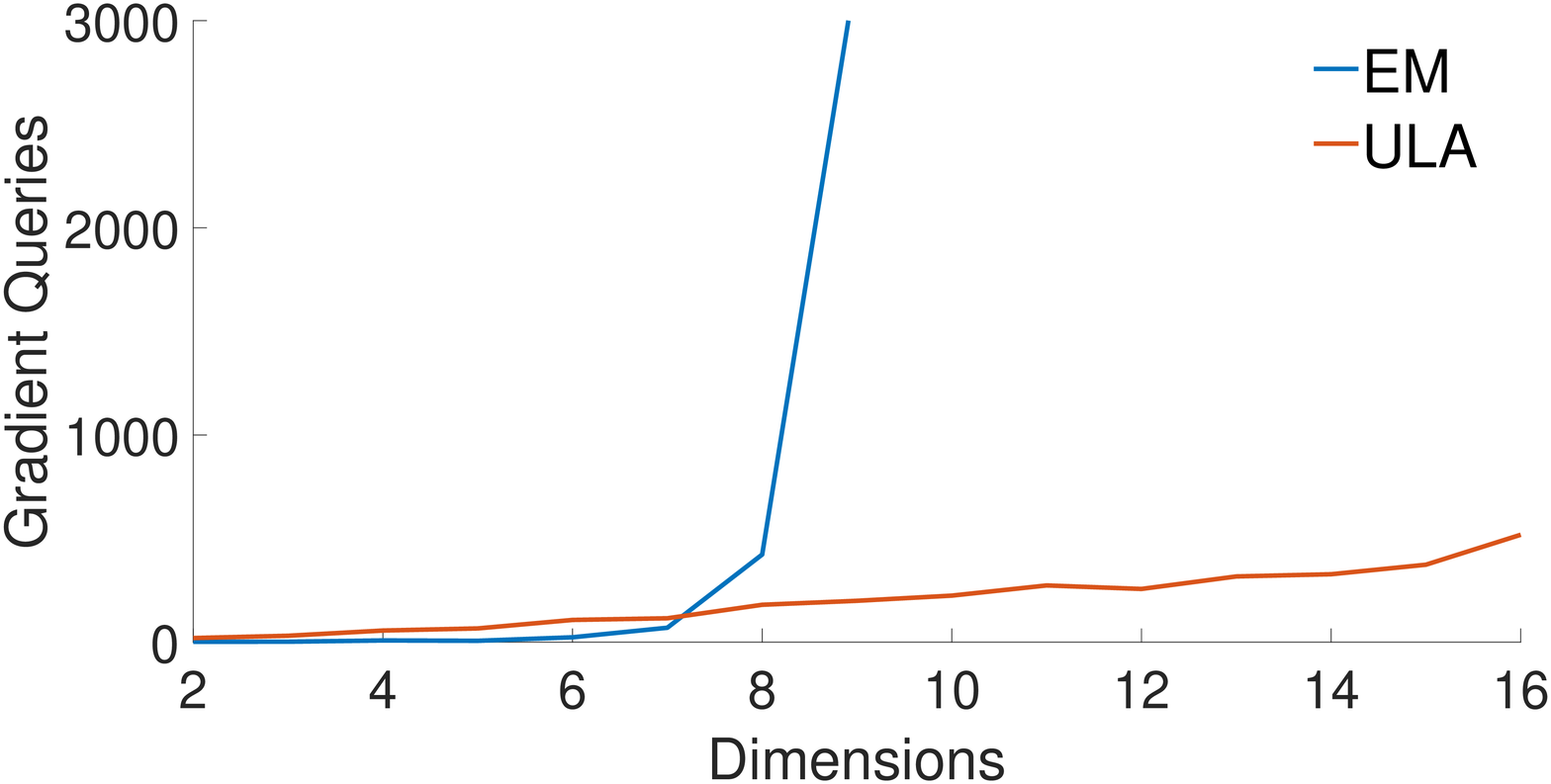}
\caption{Experimental results: scaling of number of gradient queries required for EM and ULA algorithms to converge with respect to the dimension $d$ of data.
When dimension $d\geq10$, too many gradient queries are required for EM to converge, so that an accurate estimate of convergence time is not available.
When dimension $d=32$, ULA converges within $1500$ gradient queries (not shown in the figure).
\vspace{-7pt}
}
\label{fig:experiment}
\end{figure}

Many mixture models with strongly log-concave priors fall into the assumed class of distributions with local nonconvexity. If data are distributed relatively close to each other, sampling these distributions can often be easier than searching for their global minima. This scenario is also common in the setting of the noisy multi-stable models arising in statistical physics (e.g., where the negative log likelihood is the potential energy of a classical particle system in an external field~\citep{Landau}) and related fields.

\section{Discussion}

We have shown that there is a natural family of nonconvex functions for which sampling algorithms
have polynomial complexity in dimension whereas optimization algorithms display exponential complexity.
The intuition behind these results is that computational complexity for optimization algorithms depends
heavily on the local properties of the objective function $U$.  This is consistent with a related
phenomenon that has been studied in the optimization---local strong convexity near the global optimum
can improve the convergence rate of convex optimization~\citep{Bach_Local_strong_convexity}).
On the other hand, sampling complexity depends more heavily on the global properties of $U$.
This is also consistent with existing literature; for example, it is known that the dimension
dependence of the ULA upper bounds deteriorates when $U$ changes from strongly
convex to weakly convex. This corresponds to the fact that the sub-Gaussian tails for strongly
log-concave distributions are easier to explore than the sub-exponential tails for log-concave
distributions.

A scrutiny of the relative scale between radius of the nonconvex region $R$ and the dimension
$d$ is interesting (for constant Lipschitz smoothness $L$):
when $R=0$, the problem is reduced to the Lipschitz-smooth and strongly convex
case, where GD converges in $\kappa\log(1/\epsilon)$ steps~\citep{bubeck_book}
and ULA converges in $\kappa^2 d/\epsilon^2$ steps;
when $R = \mathcal{O}(\log d)$, sampling is generally easier than optimization; when $0<R\leq\sqrt{d}$,
the convergence upper bound for sampling is still slightly smaller than the optimization complexity
lower bound; when $\sqrt{d}<R<d$, the comparison is indeterminate; and the converse is true if $R \geq d$.

The relatively rapid advance of the theory of gradient-based optimization in recent years has been
due in part to the development of lower bounds, of the kind exhibited in our Theorem~\ref{thm:lower_opt},
for broad classes of algorithms.  It is of interest to develop such lower bounds for MCMC algorithms,
particularly bounds that capture dimension dependence.  It is also of interest to develop both lower
bounds and upper bounds for other forms of nonconvexity.  For example, there has been recent work
studying strongly dissipative functions~\citep{Raginsky_dissipativity}. Here the worst-case convergence
bounds have exponential dependence on the dimension, but $p^*\propto e^{-U}$ has a sub-Gaussian tail;
thus, further exploration of this setting may yield milder conditions on $U$ that allow MCMC algorithms
to have polynomial convergence rates.

\newpage

\setcounter{theorem}{0}

\begin{appendices}
\section{Assumptions on the Objective Function $U$}
\label{assumptions}
Assumptions on $U:\R^d \rightarrow \R$ (local nonconvexity):
\begin{enumerate}
\item \label{A1}
$U(\vx)$ is $L$-Lipschitz smooth and its Hessian exists $\forall \vx\in\R^d$.

That is: $U\in C^1(\R^d)$, $\forall \vx, \vz\in\R^d$, $\lrn{\nabla U(\vx) - \nabla U(\vz)} \leq L \lrn{\vx-\vz}$;
$\forall \vx\in\R^d$, $\nabla^2 U(\vx)$ exists.

\item \label{A2}
$U(\vx)$ is $m$-strongly convex for $\lrn{\vx}>R$.

That is: $V(\vx) = U(\vx)-\dfrac{m}{2}\lrn{\vx}_2^2$ is convex on $\Omega=\R^d\setminus\ball(0,R)$\footnote{Here we let $\ball(0,R)$ denote the closed ball of radius $R$ centered at $0$.}.
We then follow the definition of convexity on nonconvex domains~\citep{Econ_convexity,MinYan} to require that $\forall \vx\in\Omega$, any convex combination of $\vx=\lambda_1 \vx_1 + \cdots + \lambda_k \vx_k$ with $\vx_1,\cdots, \vx_k\in\Omega$ satisfies:
\[
V(\vx) \leq \lambda_1 V(\vx_1) + \cdots + \lambda_k V(\vx_k).
\]
We further denote the \emph{condition number} of $U$ on $\Omega$ as $\kappa=L/m$.

\item \label{A3}
For convenience, let $\nabla U(0)=0$ (i.e., zero is a local extremum).
\end{enumerate}

\section{Proofs for Sampling}

\begin{theorem}\label{thm:MCMC}
 For $\target \propto e^{- U}$, we assume that $U$ satisfies the local nonconvexity Assumptions~\ref{A1}--\ref{A3}.
  Consider the unadjusted Langevin algorithm (ULA) and the Metropolis adjusted Langevin algorithm (MALA) with initialization $\initial=\NORMAL\parenth{0,\frac{1}{\smoothness}\Ind_\dims}$ and error tolerance $\mixingerror \in (0, 1)$.
  Then ULA satisfies
  \begin{align}
     \mixingtime_{\text{ULA}}(\mixingerror, \initial) \leq \mathcal{O}\parenth{ e^{32\smoothness R^2} \condition^2 \dfrac{\dims}{\mixingerror^2} \ln\left(\dfrac{\dims}{\mixingerror^2}\right) }. \label{thm:ULA_convergence}
  \end{align}
  For MALA,
    \begin{align}
         \mixingtime_{\text{MALA}}(\mixingerror, \initial) \leq \mathcal{O}\parenth{ \frac{e^{40\smoothness R^2}}{\scparam}\condition^{3/2}\dims^{1/2} \parenth{\dims \ln\condition + \ln\parenth{\frac{1}{\mixingerror}}}^{3/2} }. \label{thm:MALA_convergence}
     \end{align}
\end{theorem}

\begin{remark}
Assumptions~\ref{A1}--\ref{A3} can be shown to imply that the nonconvex region will have small probability mass in high dimensions. The theorem quantifies the consequences of this small mass on ULA and MALA, showing essentially that their
mixing time is not perturbed qualitatively by the nonconvexity.  It is the coupling of this result with the exponential
complexity of optimization, as shown in Theorem~\ref{theorem:lower_bound}, that is our main result.  The assumptions have been chosen to make this comparison as simple as possible.  But it is noteworthy that we can weaken the assumptions and still obtain rapid mixing for ULA and MALA.  In particular, note that we assumed that the Lipschitz parameter $L$ is uniformly bounded by a constant over the entire $\R^d$.  This assumption is in fact not necessary in our proofs. Indeed, we can allow the Lipschitz parameter $\widetilde{L}$ and strong convexity parameter $\widetilde{m}$ outside of the region $\ball(0,R)$ to scale with the dimension $d$ (while $U$ is still $L$-Lipschitz smooth inside $\ball(0,R)$ and $L$ does not scale with $d$).  In that setup, the probability mass inside the nonconvex region $\ball(0,R)$ no longer shrinks as a function of $d$.

Moreover, in that setup we can repeat the constructive proof in Lemma~\ref{lemma:hat_U} (via choosing a smaller smoothing radius $\delta = \mathcal{O}(\kappa R/d)$) and demonstrate that $\rho_U\geq L e^{-16LR^2}$.
It follows that the computational complexity for ULA becomes (in terms of dimension $d$ and accuracy $\epsilon$): $\mathcal{O}(d^3/\epsilon^2)$, where the extra $d^2$ factor is due to the fact that the step size $h$ scales inversely with $\widetilde{L}^2=\mathcal{O}(d^2)$. A similar result holds for MALA.

This more general setup highlights the value of our general approach to analyzing MCMC algorithms via the properties of weighted Sobolev spaces.  It naturally allows us to combine convergence rates for sampling strongly log-concave posteriors and those for sampling smooth posteriors in a bounded region.
Indeed, our upper bounds on convergence rates generalize existing results for strongly log-concave posteriors~\citep{Dalalyan_JRSSB,Moulines_ULA, Dalalyan_user_friendly,Xiang_underdamped,Xiang_overdamped,dwivedi2018log, Mangoubi1,Mangoubi2} and also strengthen recent work using the Wasserstein metric to the KL divergence~\citep{Eberle_ULA,Eberle_HMC,
Xiang_Nonconvex, Local_Nonconvex_W2}.
\end{remark}

We begin by proving the basic log-Sobolev inequality that underlies our results.
We then prove convergence of ULA and MALA respectively in Sec.~\ref{sec:ULA} and~\ref{sec:MALA}.

\subsection{Log-Sobolev Inequality}
\begin{proposition}
  \label{thm:log_sobolev_constant}
  For $\p^* \propto e^{- U}$ where $U$ satisfies Assumptions~\ref{A1}--\ref{A3} in Appendix~\ref{assumptions},
  \begin{align}
    \rho_U \geq \dfrac{m}{2} e^{-16LR^2}.
  \end{align}
\end{proposition}
\begin{proof}
First note that for $m/2$-strongly convex $\hat{U}\in C^1(\R^d)$ with $\nabla^2 \hat{U}(\vx)$ exists on the entire $\R^d$, distribution $e^{- \hat{U}(\vx)}$ satisfies the Bakry-Emery criterion \citep{bakryEmery} for strongly log concave density and have:
\begin{align}
\rho_{\hat{U}} \geq \dfrac{m}{2}.
\end{align}

Next we invoke Lemma~\ref{lemma:hat_U} that such $\hat{U}$ exists and satisfies $\sup\left(\hat{U}(\vx)-U(\vx)\right) - \inf\left(\hat{U}(\vx)-U(\vx)\right) \leq 16 LR^2$.

Then we use a result from Holley-Stroock \citep{HolleyStroock} and obtain:
\begin{align}
\rho_U \geq
\dfrac{m}{2} e^{- \left|\sup\left(\hat{U}(\vx)-U(\vx)\right) - \inf\left(\hat{U}(\vx)-U(\vx)\right)\right|}
\geq \dfrac{m}{2} e^{-16LR^2}.
\end{align}

\end{proof}

\begin{lemma}
For $U$ satisfying Assumptions~\ref{A1}--\ref{A3}, there exists $\hat{U}\in C^1(\R^d)$ with a Hessian that exists everywhere on $\R^d$, and $\hat{U}$ that is $m/2$-strongly convex on $\R^d$, such that $\sup\left(\hat{U}(\vx)-U(\vx)\right) - \inf\left(\hat{U}(\vx)-U(\vx)\right) \leq 16 LR^2$.
\label{lemma:hat_U}
\end{lemma}

\begin{proof}[Proof of Lemma~\ref{lemma:hat_U}]
Similar to Assumptions~\ref{A1}--\ref{A3}, denote $\Omega=\R^d\setminus\ball(0,R)$.
Also denote $\tilde U(\vx) = U(\vx)-\dfrac{m}{4}\lrn{\vx}^2$.

We follow \citep{MinYan} to construct $\hat{U}(\vx)-\dfrac{m}{4}\lrn{\vx^2} \in C^1(\R^d)$ with Hessian defined on $\R^d$ so that it is convex on $\R^d$ and differs from $\tilde U(\vx)$ less than $16 LR^2$.

First we define the function $V$ as the convex extension~\citep{roofs} of $\tilde{U}$ from domain $\Omega$ to its convex hull $\Omega^{co}$:
\begin{align}
V(\vx) = \inf_{
\substack{ \{\vx_i\} \subset \Omega,
\\ \left\{\lambda_i \big| \sum_i \lambda_i=1\right\},
\\ \text{s.t.}, \sum_i \lambda_i \vx_i = \vx }
}
\left\{\sum_{i=1}^l \lambda_i \tilde{U}(\vx_i)\right\},
\quad\forall \vx \in\Omega^{co}=\R^d.
\label{eq:V}
\end{align}
$V(\vx)$ is convex on the entire domain $\R^d$.
Also, since $\tilde{U}(\vx)$ is convex in $\Omega$, $V(\vx)=\tilde{U}(\vx)$ for $\vx\in\Omega$.
By Lemma~\ref{lemma:V}, we also know that $\forall \vx\in\ball(0,R)$, $\inf_{\bar{\vx}=R} \tilde{U}(\bar{\vx}) \leq V(\vx) \leq \sup_{\bar{\vx}=R} \tilde{U}(\bar{\vx})$.

Next we construct $\tilde{V}(\vx)$ to be a smoothing of $V$ on $\ball\left(0,\dfrac{4}{3}R\right)$.
Let $\phi\geq0$ be a smooth function supported on the ball $\ball(0,\delta)$ where $\delta=\dfrac{m}{L}\dfrac{R}{1600}<\dfrac{R}{6}$ such that $\int \phi(\vx)\rd \vx=1.$
Define
\begin{align}
\tilde{V}(\vx) = \int V(\vy)\phi(\vx-\vy)\rd \vy = \int V(\vx-\vy) \phi(\vy) \rd \vy. \label{eq:T_V}
\end{align}
Then $\tilde{V}$ is a smooth and convex function on $\R^d$.
The second expression in \eqref{eq:T_V} implies that $\tilde{V}(\vx)$ is $\dfrac{m}{2}$-strongly convex in $\R^d\setminus\ball\left(0,R+\delta\right)\supset\ball\left(0,\dfrac{3}{2}R\right)\setminus\ball\left(0,\dfrac{4}{3}R\right)$.
Also note that the definition of $\tilde{V}$ implies that $\forall \lrn{\vx}<\dfrac{4}{3}R$,
\[
\inf_{\lrn{\bar{\vx}}<\dfrac{4}{3}R+\delta} V(\bar{\vx}) \leq \tilde{V}(\vx) \leq \sup_{\lrn{\bar{\vx}}<\dfrac{4}{3}R+\delta} V(\bar{\vx}).
\]
And by Lemma~\ref{lemma:V},
\begin{align}
\inf_{\bar{\vx}\in\ball\left(0,\dfrac{4}{3}R+\delta\right)\setminus\ball(0,R)} \tilde{U}(\bar{\vx}) \leq \tilde{V}(\vx) \leq \sup_{\bar{\vx}\in\ball\left(0,\dfrac{4}{3}R+\delta\right)\setminus\ball(0,R)} \tilde{U}(\bar{\vx}),
\quad \forall \lrn{\vx}<\dfrac{4}{3}R.
\label{eq:BoundT_V}
\end{align}

Finally, we construct the auxiliary function $\hat{U}(\vx)$:
\begin{align}
\hat{U}(\vx)-\dfrac{m}{4}\lrn{\vx^2}
= \left\{
\begin{array}{l}
\tilde{U}(\vx), \quad \lrn{\vx}>\dfrac{3}{2}R\\
\alpha(\vx) \tilde{U}(\vx) + (1-\alpha(\vx))\tilde{V}(\vx), \quad \dfrac{4}{3}R<\lrn{\vx}<\dfrac{3}{2}R\\
\tilde{V}(\vx), \quad \lrn{\vx}<\dfrac{4}{3}R
\end{array}
\right.,
\end{align}
where $\alpha(\vx)=\dfrac{1}{2}\cos\left(\dfrac{36\pi}{17}\dfrac{\lrn{\vx}^2}{R^2}-\dfrac{64\pi}{17}\right)+\dfrac12$.
Here we know that $\tilde{U}(\vx)$ is $\dfrac{m}{2}$-strongly convex and smooth in $\R^d\setminus\ball\left(0,R\right)$;
$\tilde{V}(\vx)$ is $\dfrac{m}{2}$-strongly convex and smooth in $\R^d\setminus\ball\left(0,\dfrac{4}{3}R\right)$.
Hence for $\dfrac{4}{3}R<\lrn{\vx}<\dfrac{3}{2}R$,
\begin{align*}
\lefteqn{\nabla^2 \left(\hat{U}(\vx)-\dfrac{m}{4}\lrn{\vx^2}\right)} \\
&= \nabla^2 \tilde{U}(\vx) + \nabla^2\left( (1-\alpha(\vx))(\tilde{V}(\vx)-\tilde{U}(\vx)) \right)
\\&=
\alpha(\vx) \nabla^2\tilde{U}(\vx) + (1-\alpha(\vx))\nabla^2\tilde{V}(\vx)
\\&
- \nabla^2\alpha(\vx)\left(\tilde{V}(\vx)-\tilde{U}(\vx)\right)
- 2\nabla\alpha(\vx)\left(\nabla\tilde{V}(\vx)-\nabla\tilde{U}(\vx)\right)^T
\\ & \succeq
\dfrac{m}{2}\mI
- \nabla^2\alpha(\vx)\left(\tilde{V}(\vx)-\tilde{U}(\vx)\right)
- 2\nabla\alpha(\vx)\left(\nabla\tilde{V}(\vx)-\nabla\tilde{U}(\vx)\right)^T.
\end{align*}
Note that for $\dfrac{4}{3}R<\lrn{\vx}<\dfrac{3}{2}R$,
\[
\lrn{\nabla\tilde{V}(\vx)-\nabla\tilde{U}(\vx)} = \int \lrn{\nabla\tilde{U}(\vx-\vy)-\nabla\tilde{U}(\vx)} \phi(\vy) \rd \vy
\leq L\delta.
\]
\[
\tilde{V}(\vx)-\tilde{U}(\vx) = \int \left(\tilde{U}(\vx-\vy)-\tilde{U}(\vx)\right) \phi(\vy) \rd \vy
\leq \dfrac{3}{2} LR \delta.
\]
Therefore, when $\dfrac{4}{3}R<\lrn{\vx}<\dfrac{3}{2}R$,
\begin{align*}
\nabla^2 \left(\hat{U}(\vx)-\dfrac{1}{4}\lrn{\vx^2}\right)
\succeq
\dfrac{m}{2}\mI
- 3\pi\dfrac{L\delta}{R}\mI
- 54\pi^2\dfrac{L\delta}{R}\mI
\succeq
\left(\dfrac{m}{2}-800\dfrac{L\delta}{R}\right)\mI.
\end{align*}
Since $\delta=\dfrac{m}{L}\dfrac{R}{1600}$, $\nabla^2 \left(\hat{U}(\vx)-\dfrac{m}{4}\lrn{\vx^2}\right)$ is positive semi-definite for $\dfrac{4}{3}R<\lrn{\vx}<\dfrac{3}{2}R$.
Hence $\nabla^2 \left(\hat{U}(\vx)-\dfrac{m}{4}\lrn{\vx^2}\right)$ is positive semi-definite on the entire $\R^d$, and $\hat{U}(\vx)-\dfrac{m}{4}\lrn{\vx^2}$ is convex on $\R^d$.

From \eqref{eq:BoundT_V}, we know that for $\lrn{\vx}\leq \dfrac{3}{2}R$,
\[
\inf_{\bar{\vx}\in\ball\left(0,\dfrac{3}{2}R+\delta\right)\setminus\ball(0,R)} \tilde{U}(\bar{\vx})
\leq \hat{U}(\vx) - \dfrac{m}{4}\lrn{\vx^2} \leq
\sup_{\bar{\vx}\in\ball\left(0,\dfrac{3}{2}R+\delta\right)\setminus\ball(0,R)} \tilde{U}(\bar{\vx}).
\]
Therefore,
\begin{align*}
\lefteqn{\sup\left(\hat{U}(\vx)-U(\vx)\right) - \inf\left(\hat{U}(\vx)-U(\vx)\right)} \\
&= \sup\left(\hat{U}(\vx) - \dfrac{m}{4}\lrn{\vx^2} - \tilde{U}(\vx)\right) - \inf\left(\hat{U}(\vx) - \dfrac{m}{4}\lrn{\vx^2} - \tilde{U}(\vx)\right)\\
&\leq
2\left( \sup_{\bar{\vx}\in\ball\left(0,\dfrac{3}{2}R+\delta\right)\setminus\ball(0,R)} \tilde{U}(\bar{\vx})
- \inf_{\bar{\vx}\in\ball\left(0,\dfrac{3}{2}R+\delta\right)\setminus\ball(0,R)} \tilde{U}(\bar{\vx}) \right) \\
&\leq
2\left( \sup_{\bar{\vx}\in\ball\left(0,\dfrac{3}{2}R+\delta\right)} \tilde{U}(\bar{\vx})
- \inf_{\bar{\vx}\in\ball\left(0,\dfrac{3}{2}R+\delta\right)} \tilde{U}(\bar{\vx}) \right).
\end{align*}
Since $U$ is $L$-smooth, $\tilde{U}$ is $\left(L+\dfrac{m}{2}\right)$-smooth and $\nabla\tilde{U}(0)=0$.
Hence
\[
\left| \tilde{U}(\vx) - \tilde{U}(0) - \angles{ \vx, \nabla \bar{U}(0) } \right| \leq \left(\dfrac{L}{2}+\dfrac{m}{4}\right) \lrn{\vx}_2^2.
\]
So for $\forall \lrn{\vx} \leq \left(\dfrac{3}{2}R+\delta\right)$,
\[
\sup_{\bar{\vx}\in\ball\left(\dfrac{3}{2}R+\delta\right)} \tilde{U}(\bar{\vx})
- \inf_{\bar{\vx}\in\ball\left(\dfrac{3}{2}R+\delta\right)} \tilde{U}(\bar{\vx})
\leq 8LR^2.
\]
Hence
\[
\sup\left(\hat{U}(\vx)-U(\vx)\right) - \inf\left(\hat{U}(\vx)-U(\vx)\right) \leq 16 LR^2.
\]
\end{proof}

\begin{lemma}
For function $V$ defined in \eqref{eq:V}, $\forall \vx\in\ball(0,R)$, $\inf_{\lrn{\bar{\vx}}=R} \tilde{U}(\bar{\vx}) \leq V(\vx) \leq \sup_{\lrn{\bar{\vx}}=R} \tilde{U}(\bar{\vx})$.
\label{lemma:V}
\end{lemma}
\begin{proof}[Proof of Lemma~\ref{lemma:V}]
First, from the definition of $V$ inside $\ball(0,R)$:
\begin{align*}
V(\vx) &=
\inf_{ \substack{
\{\vx_i\} \subset \Omega,
\\ \left\{\lambda_i \big| \sum_i \lambda_i=1\right\}
\\ \text{s.t.}, \sum_i \lambda_i \vx_i = \vx
}
} \left\{\sum_{i=1}^l \lambda_i \tilde{U}(\vx_i)\right\}
\\ & \leq
\inf_{ \substack{
\{\vx_i\} \subset \partial\Omega,
\\ \left\{\lambda_i \big| \sum_i \lambda_i=1\right\}
\\ \text{s.t.}, \sum_i \lambda_i \vx_i = \vx
}
} \left\{\sum_{i=1}^l \lambda_i \tilde{U}(\vx_i)\right\}
\\ & \leq
\sup_{\lrn{\bar{\vx}}=R} \tilde{U}(\bar{\vx}),
\quad \forall \vx\in\ball(0,R),
\end{align*}
where the first inequality follows from the fact that $\partial\Omega\subset\Omega$ and that any $\vx\in\ball(0,R)$ can be represented as a convex combination of elements of $\partial\Omega$.

Next we prove that $\forall \vx\in\ball(0,R)$, $V(\vx)\geq \inf_{\lrn{\bar{\vx}}=R} \tilde{U}(\bar{\vx})$.
Assume that at $\vx\in\ball(0,R)$, $V(\vx)$ is equal to a linear combination of $\{\vx_i\}\subset\Omega=\R^d\setminus\ball(0,R)$: $V(\vx)= \sum_{i} \lambda_i \tilde{U}(\vx_i)$.
We hereby prove that for any $\vx_j\in\{\vx_i\}$, such that $\lrn{\vx_j}> R$, there exists a new convex combination $\{\vx_i\}\bigcup\{\bar{\vx}_j\}\setminus\{\vx_j\}$ with $\lrn{\bar{\vx}_j}=R$, such that
$V(\vx) \geq \tilde\lambda_j \tilde{U}(\bar{\vx}_j) + \sum_{i\neq j} \tilde\lambda_i \tilde{U}(\vx_i)$.

$\exists \lambda_j<\bar\lambda_j<1$, such that $\bar{\vx}_j$ defined below is a linear combination of $\vx$ and $\vx_j$ satisfying $\lrn{\bar{\vx}_j}=R$:
\[
\bar{\vx}_j = \dfrac{1-\bar\lambda_j}{1-\lambda_j} \vx + \dfrac{\bar\lambda_j - \lambda_j}{1-\lambda_j} \vx_j.
\]
Then $\bar{\vx}_j$ is a convex combination of $\{\vx_i\}$:
\[
\bar{\vx}_j = \bar\lambda_j \vx_j + \left( \dfrac{1-\bar\lambda_j}{1-\lambda_j} \right) \left(\sum_{i\neq j} \lambda_i \vx_i \right),
\]
and since $U$ is convex on $\Omega$,
\[
\tilde{U}(\bar{\vx}_j) \leq \bar\lambda_j \tilde{U}(\vx_j) + \left( \dfrac{1-\bar\lambda_j}{1-\lambda_j} \right) \left(\sum_{i\neq j} \lambda_i \tilde{U}(\vx_i) \right).
\]
On the other hand, we can reexpress $\vx$ as a convex combination of $\{\vx_i\}\bigcup\{\bar{\vx}_j\}\setminus\{\vx_j\}$:
\[
\vx = \dfrac{\lambda_j}{\bar\lambda_j} \bar{\vx}_j + \left( 1 - \dfrac{\lambda_j}{\bar\lambda_j}\dfrac{1-\bar\lambda_j}{1-\lambda_j} \right) \left(\sum_{i\neq j} \lambda_i \vx_i \right)
= \tilde\lambda_j \bar{\vx}_j +  \sum_{i\neq j} \tilde\lambda_i \vx_i,
\]
and that
\begin{align*}
V(\vx) = \sum_{i} \lambda_i \tilde{U}(\vx_i)
& \geq \dfrac{\lambda_j}{\bar\lambda_j} \tilde{U}(\bar{\vx}_j)
+ \left( 1 - \dfrac{\lambda_j}{\bar\lambda_j}\dfrac{1-\bar\lambda_j}{1-\lambda_j} \right) \left(\sum_{i\neq j} \lambda_i \tilde{U}(\vx_i) \right)
\\ & =
\tilde\lambda_j \tilde{U}(\bar{\vx}_j) +  \sum_{i\neq j} \tilde\lambda_i \tilde{U}(\vx_i).
\end{align*}
Using an inductive argument, we obtain that $\forall \vx\in\ball(0,R)$, $V(\vx)$ is bigger than or equal to a certain convex combination of $\tilde{U}(\bar{\vx}_i)$, where $\{\bar{\vx}_i\}\subset\partial\Omega$.
Therefore, $\forall \vx\in\ball(0,R)$, $V(\vx)\geq \inf_{\lrn{\bar{\vx}}=R} \tilde{U}(\bar{\vx})$.
\end{proof}

For reader's convenience, we state the Holley-Stroock lemma in the following.
\begin{lemma}[Holley-Stroock]
For probability densities $\p\propto e^{-U}$ and $\hat{\p}\propto e^{-\hat{U}}$, assume $\hat{\p}$ has log-Sobolev constant $\rho_{\hat{U}}$.
Then if $U$ is a bounded perturbation of $\hat{U}$, log-Sobolev constant $\rho_{U}$ for $\p$ satisfy:
\begin{align}
\rho_U \geq
\rho_{\hat{U}} e^{- \left|\sup\left(\hat{U}(\vx)-U(\vx)\right) - \inf\left(\hat{U}(\vx)-U(\vx)\right)\right|}.
\end{align}
\end{lemma}

\subsection{Proof of ULA Convergence Rate (Equation~\eqref{thm:ULA_convergence_main} of Theorem~\ref{thm:MCMC})}
\label{sec:ULA}

\begin{proof}[Proof of Equation~\eqref{thm:ULA_convergence_main} of Theorem~\ref{thm:MCMC}]
\label{proof_of_theorem_ula}
We first quantify the convergence of a stochastic process to a stationary distribution $\target$ via the Kullback-Leibler divergence (KL-divergence), $F(\density)$:
\begin{align*}
  F\parenth{\density} = \int \density(\vx) \ln \parenth{\frac{\p(\vx)}{\p^*(\vx)}} \rd \vx,
\end{align*}
where ${\p(\vx)}$ is absolutely continuous with respect to ${\p^*(\vx)}$; and $F\parenth{\density} = \infty$ otherwise.
Then we use the Pinsker inequality to bound the total variation norm:
\begin{align*}
  \vecnorm{\density - \target}{\text{TV}} \leq \sqrt{2 \kldiv{\density}{\target}} = \sqrt{2F\parenth{\density}},
\end{align*}
for two densities $\density$ and $\target$.

Here we take the process whose convergence is to be determined as a discretized Langevin dynamics:
\begin{align}
  \label{eq:disc_LD}
  \vX_{(k+1)h} = \vX_{kh} - \nabla U(\vX_{kh}) h + \sqrt{2} (B_{(k+1)h} - B_{hk}),
\end{align}
which is equivalent to defining for $kh < t \leq (k+1)h$:
\begin{align}
  \label{eq:disc_SDE}
  \rd \vX_t = - \nabla U(\vX_{kh}) \rd t + \sqrt{2} \rd B_t.
\end{align}
For dynamics within $kh < t \leq (k+1)h$, we have from the Girsanov theorem~\citep{SDE_book} that $\vX_t$ admits a density function $p_t$ with respect to the Lebesgue measure.
This density function can also be represented as $p_t(\vx)=\int p_{kh}(\vy) p(\vx,t|\vy,kh) \rd \vy$, where $p(\vx,t|\vy,kh)$ is
the solution to the following Kolmogorov forward equation in the weak sense~\citep{Pav_book}:
\[
\dfrac{\partial p(\vx,t|\vy,kh)}{\partial t} = \nabla^T \big(\nabla p(\vx,t|\vy,kh) + \nabla U(\vy) p(\vx,t|\vy,kh) \big),
\]
where $p(\vx,t|\vy,kh)$ and its derivatives are defined via $P_t(f) = \int f(\vx) p(\vx,t|\vy,kh) \rd \vx$ as a functional over the space of smooth bounded functions on $\mathbb{R}^d$.
It can be further established~\citep{Xiang_overdamped} that the time derivative of the KL-Divergence along $p_t$ is
\begin{align*}
\ddt F(\p_t)
= - \E{\lin{ \nabla \ln\lrp{\frac{\p_t(\vX_t)}{\p^*(\vX_t)}}, \nabla \ln p_t(\vX_t) + \nabla U(\vX_{kh}) }}, 
\end{align*}
where the expectation is taken with respect to the joint distribution of $\vX_t$ and $\vX_{kh}$.
Hence
\begin{align*}
\ddt F(\p_t)
&= - \E{\lin{ \nabla \ln\lrp{\frac{\p_t(\vX_t)}{\p^*(\vX_t)}}, \nabla \ln\lrp{\frac{\p_t(\vX_t)}{\p^*(\vX_t)}}
+ \left(\nabla U(\vX_{kh}) - \nabla U(\vX_t) \right) }} \\
&= - \E{ \lrn{ \nabla \ln\lrp{\frac{\p_t(\vX_t)}{\p^*(\vX_t)}} }^2}
+ \E{ \lin{ \nabla \ln\lrp{\frac{\p_t(\vX_t)}{\p^*(\vX_t)}}, \nabla U(\vX_t) - \nabla U(\vX_{kh}) } }.
\end{align*}
For the second term, we use Young's inequality:
\begin{align*}
\lefteqn{\E{ \lin{ \nabla \ln\lrp{\frac{\p_t(\vX_t)}{\p^*(\vX_t)}}, \nabla U(\vX_t) - \nabla U(\vX_{kh}) } }}
\\ &\leq
\dfrac12\E{ \lrn{\nabla \ln\lrp{\frac{\p_t(\vX_t)}{\p^*(\vX_t)}} }^2 }
+ \dfrac12\E{ \lrn{\nabla U(\vX_t) - \nabla U(\vX_{kh})}^2 }
\\ &\leq
\dfrac12\E{ \lrn{\nabla \ln\lrp{\frac{\p_t(\vX_t)}{\p^*(\vX_t)}} }^2 }
+ \dfrac{L^2}{2}\E{ \lrn{ \vX_t - \vX_{kh} }^2 }.
\end{align*}
Now we bound $\E{ \lrn{ \vX_t - \vX_{kh} }^2 }$ using Lipschitz smoothness of $U$
(define $\tau=t-kh\in(0,h]$):
\begin{align*}
\lefteqn{\E{ \lrn{ \vX_t - \vX_{kh} }^2 }}
\\ &\leq
\E{ \lrn{ - \nabla U(\vX_{kh}) \tau + \sqrt{2} (B_{(k+1)h} - B_{hk}) }^2 }
\\ &\leq
\Ep{\vx\sim\p_{kh}}{ \lrn{ \nabla U(\vx) }^2 } \tau^2 + 2 d \tau
\\ &\leq
\Ep{\vx\sim\p_{kh}}{ \lrn{ \vx }^2 } L^2 \tau^2 + 2 d \tau.
\end{align*}
Therefore, plugging in the bounds and using the log-Sobolev inequality proved in Proposition~\ref{thm:log_sobolev_constant}, we get for $kh < t \leq (k+1)h$:
\begin{align}
\lefteqn{\ddt F(\p_t)}
\nonumber\\ &\leq
- \dfrac12 \E{ \lrn{ \nabla \ln\lrp{\frac{\p_t(\vX_t)}{\p^*(\vX_t)}} }^2}
+ \dfrac{L^4\tau^2}{2} \Ep{\vx\sim\p_{kh}}{ \lrn{ \vx }^2 } + dL^2\tau
\nonumber\\ &=
- \dfrac12 \Ep{\vx\sim\p_t}{ \lrn{ \nabla \ln\lrp{\frac{\p_t(\vx)}{\p^*(\vx)}} }^2}
+ \dfrac{L^4\tau^2}{2} \Ep{\vx\sim\p_{kh}}{ \lrn{ \vx }^2 } + dL^2\tau
\nonumber\\ &\leq
- \rho_U F(\p_t)
+ \dfrac{L^4\tau^2}{2} \Ep{\vx\sim\p_{kh}}{ \lrn{ \vx }^2 } + dL^2\tau. \label{eq:dt_bound}
\end{align}
From Lemma~\ref{lemma:initial_dist}, we know that $\Ep{\vx\sim\p_0}{\lrn{\vx}_2^2} = \dfrac{d}{L} \leq \dfrac{16d}{\rho_U}\ln\dfrac{2L}{m} + \dfrac{512}{\rho_U} \dfrac{L^2}{m^2} LR^2$.
Combined with Lemma~\ref{lemma:variance}, we obtain that when $h\leq \dfrac{1}{4}\dfrac{\rho_U}{L^2}$, $\Ep{\vx\sim\p_{kh}}{\lrn{\vx}_2^2} \leq \dfrac{16d}{\rho_U}\ln\dfrac{2L}{m} + \dfrac{512}{\rho_U} \dfrac{L^2}{m^2} LR^2$ for any $k\in\mathbb{N}^+$.
Therefore, for $h\leq \dfrac{1}{4}\dfrac{\rho_U}{L^2}$,
\begin{align*}
\ddt F(\p_t) \leq
- \rho_U \left( F(\p_t)
- 8h^2 \dfrac{L^4}{\rho_U^2} d \ln\dfrac{2L}{m} - 256h^2 \rho_U \dfrac{L^4}{\rho_U^2} \dfrac{L^2}{m^2} LR^2 - h \dfrac{L^2}{\rho_U} d \right).
\end{align*}
Using Gronwall's inequality,
\begin{align*}
\lefteqn{F(\p_{(k+1)h})
- 8h^2 \dfrac{L^4}{\rho_U^2} d \ln\dfrac{2L}{m} - 256h^2 \rho_U \dfrac{L^4}{\rho_U^2} \dfrac{L^2}{m^2} LR^2 - h \dfrac{L^2}{\rho_U} d} \\
&\leq
e^{-\rho_U h} \left( F(\p_{kh})
- 8h^2 \dfrac{L^4}{\rho_U^2} d \ln\dfrac{2L}{m} - 256h^2 \dfrac{L^4}{\rho_U^2} \dfrac{L^2}{m^2} LR^2 - h \dfrac{L^2}{\rho_U} d \right).
\end{align*}
Therefore,
\begin{align*}
\lefteqn{F(\p_{kh})
- 8h^2 \dfrac{L^4}{\rho_U^2} d \ln\dfrac{2L}{m} - 256h^2 \dfrac{L^4}{\rho_U^2} \dfrac{L^2}{m^2} LR^2 - h \dfrac{L^2}{\rho_U} d} \\
&\leq
e^{-\rho_U hk} \left( F(\p_0)
- 8h^2 \dfrac{L^4}{\rho_U^2} d \ln\dfrac{2L}{m} - 256h^2 \dfrac{L^4}{\rho_U^2} \dfrac{L^2}{m^2} LR^2 - h \dfrac{L^2}{\rho_U} d \right) \\
&+ 8h^2 \dfrac{L^4}{\rho_U^2} d \ln\dfrac{2L}{m} + 256h^2 \dfrac{L^4}{\rho_U^2} \dfrac{L^2}{m^2} LR^2 + h \dfrac{L^2}{\rho_U} d \\
&\leq e^{-\rho_U hk} F(\p_0)
+ 8h^2 \dfrac{L^4}{\rho_U^2} d \ln\dfrac{2L}{m} + 256h^2 \dfrac{L^4}{\rho_U^2} \dfrac{L^2}{m^2} LR^2 + h \dfrac{L^2}{\rho_U} d.
\end{align*}
To make $F(\p_{kh})<\mixingerror^2$, we take:
\begin{align}
h = \dfrac{\rho_U}{4L^2} \min\left\{
\dfrac{\mixingerror^2}{d},
\sqrt{\dfrac{\mixingerror^2}{2d\ln\dfrac{2L}{m} + 64\dfrac{L^2}{m^2}LR^2}}
\right\}
= \mathcal{O}\left( e^{-16LR^2} \dfrac{m}{L^2} \cdot \min\left\{\dfrac{\mixingerror^2}{d}, \dfrac{m}{L}\dfrac{\epsilon}{\sqrt{LR^2}}\right\} \right). \label{eq:step_size}
\end{align}
Therefore, combining \eqref{eq:step_size} with Lemma~\ref{lemma:initial_dist}, we know that whenever
\begin{align}
k \geq \mathcal{O}\left( e^{32LR^2} \dfrac{L^2}{m^2} \ln\left(\dfrac{F(\p_0)}{\mixingerror^2}\right) \cdot \max\left\{\dfrac{d}{\mixingerror^2}, \dfrac{L}{m}\dfrac{\sqrt{LR^2}}{\epsilon}\right\} \right)
= \mathcal{O}\left( e^{32LR^2} \dfrac{L^2}{m^2} \ln\left(\dfrac{d}{\mixingerror^2}\right) \cdot \max\left\{\dfrac{d}{\mixingerror^2}, \dfrac{L}{m}\dfrac{\sqrt{LR^2}}{\epsilon}\right\} \right),
\end{align}
$F(\density_{kh})<\frac{1}{2}\mixingerror^2$.
Using Pinsker inequality, we obtain
\begin{align*}
  \vecnorm{\density_{kh} - \target}{\text{TV}}\leq \sqrt{2 F(\density_{kh})} \leq \mixingerror.
\end{align*}
Focusing on the dimension dependency, we obtain that the computation complexity scales as
\[
k = \mathcal{O}\left( e^{32LR^2} \dfrac{L^2}{m^2} \dfrac{d}{\mixingerror^2} \ln\left(\dfrac{F(\p_0)}{\mixingerror^2}\right) \right).
\]
\end{proof}

\begin{lemma}
For $\p_t$ following~\eqref{eq:disc_SDE}, if
$\Ep{\vx\sim\p_0}{\lrn{\vx}_2^2} \leq
\dfrac{16d}{\rho_U}\ln\dfrac{2L}{m} + \dfrac{512}{\rho_U} \dfrac{L^2}{m^2} LR^2$,
and
$h\leq \dfrac{1}{4}\dfrac{\rho_U}{L^2}$, then for all $k\in\mathbb{N}^+$,
\[
\Ep{\vx\sim\p_{kh}}{\lrn{\vx}_2^2}\leq
\dfrac{16d}{\rho_U}\ln\dfrac{2L}{m} + \dfrac{512}{\rho_U} \dfrac{L^2}{m^2} LR^2.
\]
\label{lemma:variance}
\end{lemma}

\begin{lemma}
For
\[\displaystyle
\p_0(\vx) = \left(\dfrac{L}{2\pi}\right)^{d/2} \exp\left(-\dfrac{L}{2}||\vx||^2\right)
\]
and $\p^*$ following Assumptions~\ref{A1}--\ref{A3},
\begin{align}
F(\p_0) = {\kldiv{\p_0}{\p^*}} = \int \p_0(\vx)\ln\left(\dfrac{\p_0(\vx)}{\p^*(\vx)}\right)\rd \vx \leq
\dfrac{d}{2}\ln\dfrac{2L}{m} + 32 \dfrac{L^2}{m^2} L R^2; \label{eq:KL_init}
\end{align}
\begin{align}
\Ep{\vx\sim\p_0}{\lrn{\vx}_2^2} = \dfrac{d}{L}; \label{eq:var_init}
\end{align}
and
\begin{align}
\Ep{\vx\sim\p^*}{\lrn{\vx}_2^2} \leq \dfrac{4d}{\rho_U}\ln\dfrac{2L}{m}
+ \dfrac{128}{\rho_U}\dfrac{L^2}{m^2} L R^2. \label{eq:var_end}
\end{align}
\label{lemma:initial_dist}
\end{lemma}

\subsubsection{Supporting Proofs for Equation~\eqref{thm:ULA_convergence_main} of Theorem~\ref{thm:MCMC}: Bounded Variance and $F(p_0)$}
\begin{proof}[Proof of Lemma~\ref{lemma:variance}]
Consider proof by induction.
First assume that for some $k\geq0$, for all $t=0, h, \cdots, kh$,
$\Ep{\vx\sim\p_t}{\lrn{\vx}_2^2}\leq \dfrac{16d}{\rho_U}\ln\dfrac{2L}{m} + \dfrac{512}{\rho_U} \dfrac{L^2}{m^2} LR^2$.
Then consider bounding $\Ep{\vx\sim\p_t}{\lrn{\vx}_2^2}$ for $kh < t \leq (k+1)h$, where $\p_t$ follows~\eqref{eq:disc_SDE}:
\begin{align}
\rd \vX_t = - \nabla U(\vX_{kh}) \rd t + \sqrt{2} \rd B_t.
\end{align}

To bound $\Ep{\vx_t\sim\p_t}{\lrn{\vx_t}_2^2}$, we choose an auxiliary random variable $\vx^*$ following the law of $\p^*$ and couples optimally with $x_t\sim\p_t$: $\left(\vx_t,\vx^*\right)\sim\gamma\in\Gamma_{opt}\left(\p_t,\p^*\right)$.
Then using Young's inequality, and the bound for $\Ep{\vx^*\sim\p^*}{\lrn{\vx^*}_2^2}$
\begin{align}
\Ep{\vx_t\sim\p_t}{\lrn{\vx_t}_2^2}
&=
\Ep{\left(\vx_t,\vx^*\right)\sim\gamma}{\lrn{\vx^* + (\vx_t-\vx^*)}_2^2} \nonumber\\
&\leq
2\Ep{\vx^*\sim\p^*}{\lrn{\vx^*}_2^2} + 2\Ep{\left(\vx_t,\vx^*\right)\sim\gamma}{\lrn{\vx_t-\vx^*}_2^2} \nonumber\\
&=
\dfrac{8d}{\rho_U}\ln\dfrac{2L}{m}
+ \dfrac{256}{\rho_U}\dfrac{L^2}{m^2} L R^2
+ 2W_2^2\left(\p_t,\p^*\right). \label{eq:var_to_W2}
\end{align}
Using the generalized Talagrand inequality~\citep{Villani_Talagrand} for Lipschitz smooth $\p^*$ with log-Sobolev constant $\rho_U$,
\begin{align}
W_2^2\left(\p_t,\p^*\right) \leq \dfrac{2}{\rho_U} \kldiv{\p_t}{\p^*}. \label{eq:Talagrand}
\end{align}
On the other hand, we know from \eqref{eq:dt_bound} that for $F(\p_t) = \kldiv{\p_t}{\p^*}$ (denote $\tau = t-kh$),
\begin{align*}
\ddt F(\p_t) \leq
- \rho_U F(\p_t)
+ \dfrac{L^4\tau^2}{2} \Ep{\vx\sim\p_{kh}}{ \lrn{ \vx }^2 } + dL^2\tau.
\end{align*}
Plugging in the step size $\tau \leq h \leq \dfrac{1}{4}\dfrac{\rho_U}{L^2}$ and the inductive assumption that $\Ep{\vx\sim\p_{kh}}{ \lrn{ \vx }^2 } \leq \dfrac{16d}{\rho_U}\ln\dfrac{2L}{m} + \dfrac{512}{\rho_U} \dfrac{L^2}{m^2} LR^2$, we obtain:
\begin{align*}
\ddt F(\p_t) \leq
- \rho_U F(\p_t)
+ \dfrac{\rho_U}{4} d \ln\dfrac{2L}{m} + 8 \rho_U \dfrac{L^2}{m^2} LR^2 + \dfrac{\rho_U}{4} d.
\end{align*}
Without loss of generality, assume that $L \geq 2m$.
Then
\begin{align*}
\ddt F(\p_t) \leq
- \rho_U \left( F(\p_t)
- \dfrac{d}{2} \ln\dfrac{2L}{m} - 8 \dfrac{L^2}{m^2} LR^2 \right).
\end{align*}
Using Gronwall's inequality, we obtain:
\begin{align*}
F(\p_{(k+1)h}) - \dfrac{d}{2} \ln\dfrac{2L}{m} - 8 \dfrac{L^2}{m^2} LR^2
&\leq
e^{-\rho_U h} \left( F(\p_{kh})
- \dfrac{d}{2} \ln\dfrac{2L}{m} - 8 \dfrac{L^2}{m^2} LR^2 \right) \\
&\leq
e^{-\rho_U h(k+1)} \left( F(\p_0)
- \dfrac{d}{2} \ln\dfrac{2L}{m} - 8 \dfrac{L^2}{m^2} LR^2 \right) \\
&\leq
e^{-\rho_U h(k+1)} F(\p_0) \\
&\leq
F(\p_0).
\end{align*}
Therefore, combining with~\eqref{eq:KL_init} in Lemma~\ref{lemma:initial_dist},
\begin{align}
F(\p_{(k+1)h}) &\leq
F(\p_0)
+ \dfrac{d}{2}\ln\dfrac{2L}{m} + 8 \dfrac{L^2}{m^2} LR^2 \nonumber\\
&\leq
d\ln\dfrac{2L}{m} + 40 \dfrac{L^2}{m^2} L R^2. \label{eq:KL_uniform_bound}
\end{align}
Plugging~\eqref{eq:KL_uniform_bound} into~\eqref{eq:var_to_W2} and~\eqref{eq:Talagrand},
we finish the inductive proof:
\[
\Ep{\vx\sim\p_{(k+1)h}}{\lrn{\vx}_2^2}\leq \dfrac{16d}{\rho_U}\ln\dfrac{2L}{m} + \dfrac{512}{\rho_U} \dfrac{L^2}{m^2} LR^2.
\]

\end{proof}

\begin{proof}[Proof of \eqref{eq:KL_init} of Lemma~\ref{lemma:initial_dist}]

We want to bound $F(\p_0) = \displaystyle\int \p_0(\vx)\ln\left(\dfrac{\p_0(\vx)}{\p^*(\vx)}\right)\rd \vx$, where $\p^*(\vx)\propto e^{-U(\vx)}$ and $\p_0 = \left(\displaystyle\dfrac{L}{2\pi}\right)^{d/2} \exp\left(-\dfrac{L}{2}||\vx||^2\right)$.
First define $\bar{U}(\vx) = U(\vx) - U(0)$.
Then
\[
\p^*(\vx) = \exp\left(-\bar{U}(\vx)\right) \bigg/ {\int \exp\left(-\bar{U}(\vx)\right) \rd \vx}.
\]
By Assumptions~\ref{A1} and~\ref{A3}, $\bar{U}(\vx)\leq\dfrac{L}{2}\|\vx\|^2$, $\forall \vx\in\mathbb{R}^d$.
We also prove in the following that $\bar{U}(\vx)\geq\dfrac{m}{4}\|\vx\|^2$, $\forall \vx\in\mathbb{R}^d\setminus\ball\left(0,\dfrac{8L}{m}R\right)$; and $\bar{U}(\vx)\geq-\dfrac{L}{2}\|\vx\|^2$, $\forall \vx\in\ball\left(0,\dfrac{8L}{m}R\right)$.

The latter case follows directly from Assumptions~\ref{A1} and~\ref{A3}.
For the former case, $\|\vx\|\geq \dfrac{8L}{m}R$.
Then define $\vy=\dfrac{R}{\|\vx\|} \vx$.
Since $\|\vy\| = R$,
\[
\angles{\nabla U(\vy), \vy} \geq -LR^2.
\]
Because any convex combination of $\vx$ and $\vy$ belongs to the set $\mathbb{R}^d\setminus\ball(0,R)$, where $U$ is $m$-strongly convex,
\begin{align*}
U(\vx) - U(\vy)
&\geq
\angles{\nabla U(\vy), \vx-\vy} + \dfrac{m}{2}\|\vx-\vy\|^2 \\
&=
\left(\dfrac{\|\vx\|}{R}-1\right)\angles{\nabla U(\vy), \vy} + \dfrac{m}{2} \left(\dfrac{\|\vx\|}{R}-1\right)^2 \\
&\geq - \left(\dfrac{\|\vx\|}{R}-1\right) LR^2 + \dfrac{m}{2} \left(\dfrac{\|\vx\|}{R}-1\right)^2 \\
&\geq \dfrac{m}{4} \|\vx\|^2 + LR^2,
\end{align*}
since $\|\vx\|\geq \dfrac{8L}{m}R$.
Again, using Assumptions~\ref{A1} and~\ref{A3}, $U(\vy)\geq-\dfrac{L}{2}R^2$, which leads to the result that
$U(\vx)\geq \dfrac{m}{4} \|\vx\|^2$.

Therefore, $U(\vx)\geq \dfrac{m}{4} \|\vx\|^2 - 32 \dfrac{L^2}{m^2} L R^2$ and
\begin{align*}
-\ln p^*(\vx) &= \bar{U}(\vx) + \ln{\int \exp\left(-\bar{U}(\vx)\right) \rd \vx} \\
&\leq \dfrac{L}{2}\|\vx\|^2 + \ln \int \exp\left(-\dfrac{m}{4}\|\vx\|^2 + 32 \dfrac{L^2}{m^2} L R^2\right) \rd \vx \\
&= \dfrac{L}{2}\|\vx\|^2 + \dfrac{d}{2}\ln\dfrac{4\pi}{m} + 32 \dfrac{L^2}{m^2} L R^2.
\end{align*}
Hence
\begin{align*}
- \int \p_0(\vx) \ln p^*(\vx) \rd \vx
\leq 32 \dfrac{L^2}{m^2} L R^2 + \dfrac{d}{2}\ln\dfrac{4\pi}{m} + \dfrac{d}{2}.
\end{align*}
We can also calculate that
\begin{align*}
\int \p_0(\vx) \ln p_0(\vx) \rd \vx
= -\dfrac{d}{2}\ln\dfrac{2\pi}{L} - \dfrac{d}{2}.
\end{align*}
Therefore,
\begin{align*}
F(\p_0) &= \int \p_0(\vx) \ln p_0(\vx) \rd \vx - \int \p_0(\vx) \ln p^*(\vx) \rd \vx \\
&\leq 32 \dfrac{L^2}{m^2} L R^2 + \dfrac{d}{2}\ln\dfrac{2L}{m}.
\end{align*}

\end{proof}

\begin{proof}[Proof of \eqref{eq:var_init} of Lemma~\ref{lemma:initial_dist}]
It is straightforward to calculate that $\Ep{\p_0}{\|\vx\|_2^2} = {\rm trace}\left(\dfrac{1}{L}\mI\right) = \dfrac{d}{L}$.

It is worth noting that the choice of the initial condition $\p_0$ can be flexible.
For example, if we choose $\vx_0\sim\mathcal{N}\left(0,\dfrac{1}{m}\mI\right)$, then $F(\p_0)\leq 32 \dfrac{L^2}{m^2} L R^2 + \dfrac{d}{2}\cdot\dfrac{L}{m}$ and $\Ep{\p_0}{\|\vx\|_2^2} = \dfrac{d}{m} \leq 48 R^2 + \dfrac{4d}{m}$ (resulting in merely an extra $\log\dfrac{L}{m}$ term in the computation complexity).
\end{proof}

\begin{proof}[Proof of \eqref{eq:var_end} of Lemma~\ref{lemma:initial_dist}]
To bound $\Ep{\vx^*\sim\p^*}{\lrn{\vx^*}_2^2}$, we choose an auxiliary random variable $\vx_0$ following the law of $\p_0$ and couples optimally with $x^*\sim\p^*$: $\left(\vx^*,\vx_0\right)\sim\gamma\in\Gamma_{opt}\left(\p^*,\p_0\right)$.
Then using Young's inequality,
\begin{align*}
\Ep{\vx^*\sim\p^*}{\lrn{\vx^*}_2^2}
&=
\Ep{\left(\vx^*,\vx_0\right)\sim\gamma}{\lrn{\vx_0 + (\vx^*-\vx_0)}_2^2} \\
&\leq
2\Ep{\vx_0\sim\p_0}{\lrn{\vx_0}_2^2} + 2\Ep{\left(\vx^*,\vx_0\right)\sim\gamma}{\lrn{\vx^*-\vx_0}_2^2} \\
&=
\dfrac{2d}{L} + 2W_2^2\left(\p^*,\p_0\right).
\end{align*}
Using the generalized Talagrand inequality~\citep{Villani_Talagrand} for Lipschitz smooth $\p^*$ with log-Sobolev constant $\rho_U$,
\[
W_2^2\left(\p^*,\p_0\right) \leq \dfrac{2}{\rho_U} \kldiv{p_0}{p^*}.
\]
On the other hand, we know from \eqref{eq:KL_init} that
\[
\kldiv{p_0}{p^*} \leq \dfrac{d}{2}\ln\dfrac{2L}{m} + 32 \dfrac{L^2}{m^2} L R^2.
\]
Therefore,
\begin{align*}
\Ep{\vx^*\sim\p^*}{\lrn{\vx^*}_2^2}
&\leq
\dfrac{2d}{L}
+ \dfrac{2d}{\rho_U}\ln\dfrac{2L}{m} + \dfrac{128}{\rho_U} \dfrac{L^2}{m^2} L R^2 \\
&\leq \dfrac{4d}{\rho_U}\ln\dfrac{2L}{m}
+ \dfrac{128}{\rho_U}\dfrac{L^2}{m^2} L R^2.
\end{align*}
\end{proof}


\subsection{Proof of MALA Convergence Rate (Equation~\eqref{thm:MALA_convergence_main} of Theorem~\ref{thm:MCMC})}
\label{sec:MALA}

\begin{proof}[Proof of Equation~\eqref{thm:MALA_convergence_main}] 
\label{proof_of_theorem_mala}
Our proof of Equation~\ref{thm:MALA_convergence_main} is based on the following two lemmas. The first one characterizes the convergence of MALA under a warm starting distribution. The second one shows that the initial distribution $\NORMAL\parenth{0, \frac{1}{2\smoothness} \Ind_\dims}$ is $\mathcal{O}(e^\dims)$-warm. Let us first define the warm start.
\begin{definition}[Warm start]
    \label{def:warm_sart}
    Given a scalar $\theta > 0$, an initial distribution with density $\initial$ is said to be $\warmparam$-warm with respect to the stationary distribution with density $\target$ if
    \begin{align*}
       \forall \vx \in \real^\dims, \frac{\initial(\vx)}{\target(\vx)} \leq \warmparam.
    \end{align*}
\end{definition}
\begin{lemma}
    \label{lem:MALA_convergence_warm_init}
    Assume $\p^*(\vx) \propto e^{- U(\vx)}$ where $U$ satisfies the local nonconvexity Assumptions~\ref{A1}--\ref{A3}. Then the MALA with a $\warmparam$-warm distribution with density $\initial$ and error tolerance $\mixingerror \in (0, 1)$, satisfies
    \begin{align}
         \mixingtime(\mixingerror, \initial) \leq \mathcal{O}\parenth{ \frac{e^{32\smoothness R^2}}{\scparam}\cdot \ln\parenth{\frac{2\warmparam}{\mixingerror}} \cdot \max\braces{r\parenth{\frac{2\warmparam}{\mixingerror}} \condition^{3/2}\dims^{1/2}, \condition \dims}}.
     \end{align}
\end{lemma}
\begin{lemma}
    \label{lem:normal_is_warm_start}
    The initial distribution $\NORMAL\parenth{0, \frac{1}{\smoothness}\Ind_\dims}$ is $e^{16\smoothness R^2} \parenth{2\condition }^{\dims/2}$-warm with respect to the target distribution $\target$.
\end{lemma}
Equation~\ref{thm:MALA_convergence_main} directly follows by combining Lemma~\ref{lem:MALA_convergence_warm_init} and Lemma~\ref{lem:normal_is_warm_start}.
\end{proof}

\begin{proof}[Proof of Lemma~\ref{lem:MALA_convergence_warm_init}]
At a high level, the proof closely follows the proof of Theorem 1 in~\citep{dwivedi2018log}. We replace their Lemma 1 with Lemma~\ref{lem:distribution_distance} to establish that for an appropriate choice of stepsize, the MALA updates have large overlap inside the high probability ball. Lemma~\ref{lem:flow_lower_bound} allows us to obtain a lower bound on the conductance. Finally applying the Lovasz lemma, we obtain convergence guarantees.

 In order to start the proof, we first introduce conductance related notions for a general Markov chain. Consider an ergodic Markov chain defined by a transition operator $\transition$, and let $\stationary$ denote its stationary distribution. We define the ergodic flow from $A$ to its complement $A^c$
\begin{align*}
    \flow (A) = \int_A \transition_{\mathbf{u}} (A^c) \target(\mathbf{u}) \rd\mathbf{u}.
\end{align*}
For each scalar $s \in (0, 1/2)$, we define the $s$-conductance
\begin{align*}
    \conductance_s = \inf_{\stationary(A) \in (s, 1-s)} \frac{\flow (A)}{\min\braces{\stationary(A) -s, \stationary(A^c) - s}}.
\end{align*}
The notation $\transition_\mathbf{u}$ is the shorthand for the distribution $\transition(\delta_{\mathbf{u}})$ obtained by applying the transition operator to a dirac distribution concentrated on $\mathbf{u}$.

For a Markov chain with $\theta$-warm start initial distribution $\Pi_0$, having $s$-conductance $\conductance_s$, Lov\a'asz and Simonovits~\citep{lovasz1993random} proved its convergence
\begin{align}
    \label{eq:lovasz_convegence_s_conductance}
    \vecnorm{\transition^k(\Pi_0) - \stationary}{\text{TV}} \leq \theta s + \theta \parenth{1 - \frac{\conductance_s^2}{2}}^k \leq \theta s + \theta e^{-k\conductance_s^2/2} \text{ for any } s \in (0, \frac{1}{2}).
\end{align}
We will apply this result for $s$ small by cutting off the probability mass outside a Euclidean ball. We define radius
\begin{align}
    \label{eq:radius_constant_r}
    \radius(s) = 2 + \sqrt{2} e^{8\smoothness R^2} \ln^{0.5}\parenth{\dims/s} + 7R/\sqrt{\dims/\scparam},
\end{align}
and the Euclidean ball
\begin{align}
    \label{eq:truncball}
    \truncball_s = \mathbb{B}\parenth{0, \radius(s) \sqrt{\frac{\dims}{\scparam}}}.
\end{align}
We define the appropriate stepsize.
\begin{align}
    \stepbound(s, \tvscalar) &= \min\bigg\{
\frac{\sqrt\tvscalar}{8\sqrt{2} \radius(s)}
\frac{\sqrt{\scparam}}{\smoothness\sqrt{\dims\smoothness}},\quad
\frac{\tvscalar}{96\scalar_\tvscalar}\frac{1}{\smoothness\dims}, \quad
\frac{\tvscalar^{2/3}}{26(\scalar_\tvscalar \radius^2(s) )^{1/3}}
\frac{1}{\smoothness}\parenth{\frac{\scparam}{\smoothness\dims^2}}^{1/3}
\bigg\},\\
\text{where}\quad \scalar_\tvscalar &= 1+2\sqrt{\log(16/\tvscalar)} +
2\log(16/\tvscalar).
\end{align}
Applying Lemma~\ref{lem:distribution_distance} with $\step = \stepbound(s, \gamma)$, for $\vx, \vy \in \truncball_s$ and $\vecnorm{\vx-\vy}{2} \leq \Delta = \gamma \sqrt{h} / 4$, we obtain
\begin{align}
    \label{eq:concrete_overlap}
    \vecnorm{\transition_\vx-\transition_\vy}{\text{TV}}
    & \leq \vecnorm{\transition_\vx-\proposal_\vx}{\text{TV}} + \vecnorm{\proposal_\vx-\proposal_\vy}{\text{TV}} + \vecnorm{\proposal_\vy-\transition_\vy}{\text{TV}} \notag \\
    & \leq \frac{\sqrt{2}\gamma }{4} + \frac{\gamma}{8} +  \frac{\sqrt{2}\gamma }{4} \notag \\
    & \leq \gamma.
\end{align}
Applying Lemma~\ref{lem:flow_lower_bound} with $\convexset = \truncball_s$ in combination with Lemma~\ref{lem:mass_truncball}, Lemma~\ref{lem:log_sobolev_implies_isoperimetric} and Lemma~\ref{lem:distribution_distance}, we obtain that for stepsize $h \in (0, \stepbound(s, \gamma)]$, the $s$-conductance is lower bounded.
\begin{align*}
    \conductance_s \geq \frac{\parenth{1-\gamma} \cdot \parenth{1-s}^2 \cdot \gamma\sqrt{\step} \cdot \logsobconstU}{256}.
\end{align*}
Now we can conclude by making appropriate choice of $s$ and $\gamma$. Letting $s =\dfrac{\mixingerror}{2\warmparam}$ and $\gamma = \dfrac{1}{2}$, we obtain
\begin{align*}
    \conductance_s \geq \mathcal{O}(\logsobconstU\cdot\sqrt{\step}).
\end{align*}

Plugging this conductance expression into the result of Lov\a'asz and Simonovits \eqref{eq:lovasz_convegence_s_conductance}, with $\Pi_0$ the distribution with density $\initial$ and $\stationary$ the stationary distribution with density $\target$, we obtain that
\begin{align*}
    \vecnorm{\transition^k(\Pi_0) - \stationary}{\text{TV}} \leq \warmparam \frac{\mixingerror}{2 \warmparam} + \warmparam e^{-k\conductance_s^2/2} \leq \mixingerror, \text{ for } k \geq \mathcal{O}\parenth{\frac{1}{\logsobconstU^2\step} \cdot \ln\parenth{\frac{2\warmparam}{\mixingerror}}},
\end{align*}
where
\begin{align*}
    \logsobconstU \geq \frac{\scparam}{2} e^{-16 \smoothness R^2}, \text{ and } \quad h = \mathcal{O}\parenth{\min\braces{\frac{1}{\smoothness \cdot \radius(\frac{2\warmparam}{\mixingerror})\kappa^{1/2} \dims^{1/2}},  \frac{1}{\smoothness \dims}}}.
\end{align*}
This concludes the proof of this lemma.
\end{proof}

\begin{lemma}
    \label{lem:mass_truncball}
    For any $s \in (0, \frac{1}{2})$, we have $\stationary(\truncball_s) \geq 1 - s$.
\end{lemma}

\begin{lemma}
    \label{lem:log_sobolev_implies_isoperimetric}
    If the density $\target$ satisfies the log-Sobolev inequality with constant $\logsobconstU$, then it also satisfies the following isoperimetric inequality with constant $\logsobconstU$: For any $A$ and $B$ open disjoint subsets of $\real^\dims$, $C = \real^\dims \setminus (A \cup B)$, $\stationary$ being the probability measure for $\target$, we have
    \begin{align}
        \label{eq:isoperimetric_ineq}
        \stationary(A) \geq \logsobconstU \cdot d(A, B) \stationary(A) \stationary(B),
    \end{align}
    where $d(A, B) = \min_{\vx \in A, y \in B} \vecnorm{\vx - \vy}{2}$, is the set distance with Euclidean metric on $\real^\dims$.
\end{lemma}

\begin{lemma}
    \label{lem:flow_lower_bound}
    Let $\convexset$ be a convex set such that $\vecnorm{\transition_\vx-\transition_\vy}{\text{TV}} \leq \gamma$ whenever $\vx, \vy \in \convexset$ and $\vecnorm{\vx-\vy}{2} \leq \Delta$. $\stationary$ satisfies the partition type isoperimetric inequality \eqref{eq:isoperimetric_ineq} with constant $\rho$. Then for any measurable partition $A_1$ and $A_2$ of $\real^\dims$, we have
    \begin{align}
         \label{eq:flow_lower_bound}
         \int_{A_1} \transition_{\mathbf{u}}(A_2) \target(\mathbf{u}) \rd\mathbf{u} \geq \frac{\rho}{8} \min\braces{1, \frac{\Delta\cdot (1-\gamma) \cdot \stationary^2(\convexset)}{8}} \min\braces{\stationary(A_1 \cap \convexset), \stationary(A_2 \cap \convexset)}.
     \end{align}
\end{lemma}

\begin{lemma}
    \label{lem:distribution_distance}
\begin{subequations}
For any step size $\step \in \big(0, \frac{1}{\smoothness} \big]$, the MALA proposal distribution satisfies the bound
\begin{align}
    \label{eq:proposal_difference}
    \sup_{ \substack{\vx,\vy \in \real^\dims \\ \vx \neq \vy}}
    \frac{\vecnorm{\proposal_\vx^{\tagmala(\step)} -
        \proposal_\vy^{\tagmala(\step)}}{\text{TV}}}{\vecnorm{\vx -\vy}{2}} &\leq
    \sqrt{\frac{2}{\step}}.
\end{align}
Moreover, given scalars $s \in (0, 1/2)$ and $\gamma \in (0,
1)$, then the MALA proposal distribution for any stepsize $\step \in
\big(0, \stepbound(s, \tvscalar) \big]$ satisfies the bound
\begin{align}
    \label{eq:transition_difference}
    \sup_{\vx \in \truncball_s} \vecnorm{\proposal_\vx^{\tagmala(\step)} -
      \transition_\vx^{\tagmala(\step)}}{\text{TV}} \leq \frac{\gamma}{8},
\end{align}
where the truncated ball $\truncball_s$ was defined in \eqref{eq:truncball}.
\end{subequations}
\end{lemma}

\begin{remark}
It can be seen that the constraint on the step size $h$ originates from \eqref{eq:transition_difference}, where the difference between the proposal and transition distributions are bounded by the acceptance rate (see proof of \eqref{eq:transition_difference}).
The resulting step size scaling with respect to the dimension $d$ is $h=\tilde{w}(s,\gamma)=\mathcal{O}(d^{-1})$ under our current assumption.
In a celebrated work~\citep{Gareth_optimal_scaling}, with extra assumptions on higher order smoothness and decomposability of the target distribution $p^*$, the log-acceptance rate was expanded to higher orders and a much better scaling of $h=\mathcal{O}(d^{-1/3})$ was obtained.
It would be of great theoretical interest to understand whether such scaling can be achieved without the decomposability assumption on $p^*$.
\end{remark}

\subsubsection{Supporting Proofs for Equation~\eqref{thm:MALA_convergence_main} of Theorem~\ref{thm:MCMC}} 
\label{ssub:proof of auxillary_lemmas_for_theorem_thm:mala_convergence}
\begin{proof}[Proof of Lemma~\ref{lem:normal_is_warm_start}]
The starting distribution $\NORMAL\parenth{0, \frac{1}{\smoothness}\Ind_\dims}$ has density
\begin{align*}
    \initial(\vx) = \parenth{\frac{\smoothness}{2\pi}}^{\dims/2} e^{-\dfrac{\smoothness \lrn{\vx}^2}{2}}.
\end{align*}
Taking the ratio with respect to the stationary distribution, we have
\begin{align*}
    \frac{\initial(\vx)}{\target(\vx)} &= \frac{\initial(\vx)}{\dfrac{1}{\int e^{-U(\vx)} \rd\vx} e^{-U(\vx)}} \\
    & \leq e^{16\smoothness R^2} \parenth{2\condition }^{\dims/2} \cdot \exp\parenth{-\smoothness \lrn{\vx}^2/2 + U(\vx)} \\
    & \leq e^{16\smoothness R^2} \parenth{2\condition }^{\dims/2}.\\
\end{align*}
The first inequality is because, according to Lemma~\ref{lemma:hat_U}, we have
\begin{align*}
    \int e^{-U(\vx)} \rd\vx \leq e^{16 \smoothness R^2} \cdot \int e^{-\dfrac{\scparam \lrn{\vx}^2}{4}} \rd\vx = e^{16 \smoothness R^2} \parenth{\frac{\scparam}{4\pi}}^{\dims/2}.
\end{align*}
\end{proof}

\begin{proof}[Proof of Lemma~\ref{lem:mass_truncball}]
This lemma relies on the concentration of the stationary distribution $\target$ around $0$. The concentration follows from the log-Sobolev constant shown in Proposition~\ref{thm:log_sobolev_constant}. The following lemma is a classical way to obtain concentration from the log-Sobolev inequality is based on Herbst argument (e.g. see Section 2.3 in~\citep{ledoux1999concentration}).
\begin{lemma}
    \label{eq:herbst_logsobolev_concentration}
    If $\stationary$ satisfies a log-Sobolev inequality with constant $\rho$ then every $1$-Lipschitz function $f$ is integrable with respect to $\stationary$ and satisfies the concentration inequality
    \begin{align*}
        \Prob_{\vx \sim \stationary}\brackets{f(\vx) > \Exs_\stationary\brackets{f} + t} \leq e^{- \rho t^2/2}.
    \end{align*}
\end{lemma}
Applying this lemma with $f$ being the projection to each coordinate and using union bound, we obtain that
\begin{align*}
    \Prob_{\vx \sim \stationary}\brackets{ \vecnorm{\vx - \Exs[\vx]}{2}^2 > \frac{2t\dims}{\logsobconstU}} \leq \dims e^{-t}.
\end{align*}
We define $\mathcal{B}_1 = \mathbb{B}\parenth{\Exs[\vx], \sqrt{2 \log(\frac{\dims}{s})\frac{\dims}{\logsobconstU}}}$. Taking $t = \log(\frac{\dims}{s})$, we obtain that
\begin{align*}
    \stationary(\mathcal{B}_1) \geq 1 - s.
\end{align*}
Using the results proved in Lemma~\ref{lemma:variance}, we can also turn this concentration around the mean to the concentration around $0$. According to Lemma~\ref{lemma:variance}, we have
\begin{align*}
    \Exs_{\vx \sim \stationary} \vecnorm{\vx}{2}^2 \leq 48 R^2 + \frac{4\dims}{\scparam}.
\end{align*}
Using Jensen's inequality, we obtain
\begin{align*}
    \vecnorm{\Exs_{\vx\sim \stationary}[\vx]}{2} \leq \Exs_{\vx\sim \stationary} \vecnorm{\vx}{2} \leq \sqrt{\Exs_{\vx\sim \stationary} \vecnorm{\vx}{2}^2} \leq \sqrt{48 R^2 + \frac{4\dims}{\scparam}}.
\end{align*}
We define $\mathcal{B}_2 = \mathbb{B}\parenth{0, \sqrt{48 R^2 + \frac{4\dims}{\scparam}} + \sqrt{2 \log(\frac{\dims}{s})\frac{\dims}{\logsobconstU} }}$. We deduce that
\begin{align*}
    \mathcal{B}_1 \subset \mathcal{B}_2 \subset \truncball_s.
\end{align*}
As a result, we obtain $\stationary(\truncball_s) \geq \stationary(\mathcal{B}_1) \geq 1 -s$ as claimed.
\end{proof}

\begin{proof}[Proof of Lemma~\ref{lem:log_sobolev_implies_isoperimetric}]
Lemma~\ref{lem:log_sobolev_implies_isoperimetric} shows that log-Sobolev inequality implies isoperimetric inequality with constants of the same order. It is pretty standard. Since we can't find a complete proof in the literature, we provide it for completeness. $\target$ satisfies the following log-Sobolev inequality, for any smooth $g: \real^\dims \rightarrow \real$.
\begin{align}
    \label{eq:log-Sobolev_ineq}
    2 \logsobconstU \brackets{\int_{\real^\dims} g \ln g \rd\stationary - \int_{\real^\dims} g \rd\stationary \cdot \ln\parenth{\int_{\real^\dims} g \rd\stationary}} \leq \int_{\real^\dims} \frac{\vecnorm{\nabla g}{2}^2}{g} \rd\stationary.
\end{align}
where
\begin{align*}
    \rd\stationary(\vx) = \target(\vx) \rd\vx.
\end{align*}
Replacing $g$ with $g^2$ in \eqref{eq:log-Sobolev_ineq}, for $g: \real^\dims \mapsto \real$, we obtain the equivalent form
\begin{align*}
    2 \rho_U \text{Ent}\parenth{g^2} \leq \int_{\real^\dims} \vecnorm{\nabla g}{2}^2 \rd\stationary,
\end{align*}
where
\begin{align*}
    \text{Ent}_{\target}(g^2) = \brackets{\int_{\real^\dims} g^2 \ln g^2 \rd\stationary - \int_{\real^\dims} g^2 \rd\stationary \cdot \ln\parenth{\int_{\real^\dims} g^2 \rd\stationary}}.
\end{align*}

It is well known that the log-Sobolev inequality implies the following Poincar\'e inequality with the same constant (e.g. \citep{bobkov2006modified}). For any smooth $g: \real^\dims \rightarrow \real$, we have
\begin{align}
    \label{eq:Poincare}
    \logsobconstU \text{Var}_{\target}(g) \leq \int_{\real^\dims} \vecnorm{\nabla g}{2}^2 \rd\stationary,
\end{align}
where
\begin{align*}
    \text{Var}_{\target}(g) = \int_{\real^\dims} g^2 \rd\stationary - \parenth{\int_{\real^\dims} g \rd\stationary}^2.
\end{align*}
This implication is based on the fact that the gradient operator is invariant to translation (i.e. for $c \in \real$, $\nabla(f+c) = \nabla f $) and
\begin{align*}
    \text{Ent}\parenth{\parenth{f+c}^2} \to 2 \text{Var}(f), \text{ as } c \to \infty.
\end{align*}
Next, we show that the isoperimetric constant can by lower bounded by the Poincar\'e constant. We denote $\Psi$ the isoperimetric constant defined as
\begin{align}
    \label{eq:isoperimetric_ineq_grad_form}
    \Psi = \sup_{A \subset \real^\dims, \text{ open}} \frac{\stationary^+(\partial A)}{\min{\stationary(A), 1 - \stationary(A)}},
\end{align}
where $\stationary^+(\partial A) = \lim_{\delta \to 0}\frac{\stationary\parenth{A + \delta} - \stationary\parenth{A}}{\delta}$.
Taking a sequence of smooth $\braces{g_k}_{k=1, \ldots \infty}$ with limit the indicator function of $A$ in equation~\eqref{eq:Poincare}, we obtain\footnote{Note that Buser's inequality~\citep{buser1982note} (Theorem 1.2), which would give $h \geq \parenth{\logsobconstU/10}^{1/2}$, does not directly apply here because of the possible negative curvature.}
\begin{align*}
    \Psi \geq \logsobconstU.
\end{align*}
Finally, it is easy to show that the infinitesimal version of the isoperimetric inequality in \eqref{eq:isoperimetric_ineq_grad_form} is equivalent to the partition version (see e.g.~\citep{bobkov2005entropy} Proposition 11.1 and~\citep{bobkov2007isoperimetric}). Let $A$ and $B$ be open disjoint subsets of $\real^\dims$, $C = \real^\dims \setminus \parenth{A \cup B}$, then
\begin{align}
    \label{eq:isoperimetric_ineq_partition_form}
    \stationary(C) \geq \logsobconstU \cdot d(A, B) \stationary(A) \stationary(B).
\end{align}

\end{proof}

In the following, we provide useful lemmas for proving Lemma~\ref{lem:MALA_convergence_warm_init}.


\begin{proof}[Proof of Lemma~\ref{lem:flow_lower_bound}]
The proof of this lemma follows directly from the proof of Lemma 2 in~\citep{dwivedi2018log}. The main difference in the setting is that the target distribution is no longer log-concave, however, the proof follows because the log-concavity was never used in the proof of this lemma. It is sufficient to replace the isoperimetric inequality with ours in \eqref{eq:isoperimetric_ineq_partition_form}.
\end{proof}

\begin{proof}[Proof of Lemma~\ref{lem:distribution_distance}]
We prove the two claims in this Lemma separately. In order to simplify notation, we drop the superscript from our notations of distributions $\transition_\vx^{\tagmala(\step)}$ and $\proposal_\vx^{\tagmala(\step)}$.
\end{proof}

\begin{proof}[Proof of \eqref{eq:proposal_difference}]
We first apply the Pinsker inequality~\citep{cover2012elements} to bound the total variation distance via KL-divergence.
\begin{align*}
    \vecnorm{\proposal_\vx - \proposal_\vy}{\text{TV}} \leq \sqrt{2 \kldiv{\proposal_\vx}{\proposal_\vy}}.
\end{align*}
Since our proposals before applying Metropolis filters follow multivariate normal distributions, we obtain closed form expressions for the KL divergence.
\begin{align*}
    \vecnorm{\proposal_\vx - \proposal_\vy}{\text{TV}}
    &\leq \sqrt{2 \kldiv{\proposal_\vx}{\proposal_\vy}} \\
    & = \frac{\vecnorm{\stationary_\vx - \stationary_\vy}{2}}{\sqrt{2\step}} \\
    & = \frac{\vecnorm{\parenth{\vx - \step\nabla \func(\vx)} - \parenth{\vy - \step\nabla \func(\vy)}}{2}}{\sqrt{2\step}}. \\
\end{align*}
Here we use the smoothness without using the convexity to bound the last term. We have
\begin{align*}
    \vecnorm{\parenth{\vx - \step \nabla \func(\vx) } - \parenth{\vy - \step \nabla \func(\vy)}}{2}
    &=  \vecnorm{\int_0^1 \brackets{\Ind_\dims - \step \nabla^2 \func(\vx + t(\vx-\vy))}(\vx- \vy) \rd t}{2} \\
    &\leq \int_0^1 \vecnorm{\brackets{\Ind_\dims - \step \nabla^2 \func(\vx + t(\vx-y))}(\vx- \vy)}{2} \rd t \\
    &\leq \sup_{t \in [0, 1]} \matsnorm{\Ind_\dims  - \step \nabla^2 \func(\vx+t(\vx-\vy))}{2} \vecnorm{\vx - \vy}{2} \\
    &\leq 2 \vecnorm{\vx - \vy}{2}.
\end{align*}
The last inequality follows from the fact that $\nabla^2 \func(\vz) \preceq \smoothness \Ind_\dims $ for all $\vz\in \real^\dims$. Note that we lose a 2 factor without using the convexity.
\end{proof}

\begin{proof}[Proof of \eqref{eq:transition_difference}]
We denote $p_\vx$ the density corresponding to the proposal distribution $\proposal_\vx = \NORMAL\parenth{\vx - \step \nabla \func(\vx), 2 \step \Ind_\dims}$. We have
\begin{align*}
    \vecnorm{\proposal_\vx - \transition_\vx}{\text{TV}}
    & = \frac{1}{2}\parenth{\transition_\vx\parenth{\braces{\vx}} + \int_{\real^\dims} p_\vx(\vz) \rd\vz - \int_{\real^\dims} \min\braces{1, \frac{\target(\vz) \cdot p_\vz(\vx)}{\target(\vx) \cdot p_\vx (\vz)}} p_\vx(\vz) \rd\vz } \\
    & = \frac{1}{2}\parenth{2 - 2 \int_{\real^\dims} \min\braces{1, \frac{\target(\vz) \cdot p_\vz(\vx)}{\target(\vx) \cdot p_\vx (\vz)}} p_\vx(\vz) \rd\vz} \\
    & \leq 1 - \Exs_{\vz \sim \proposal_\vx}\brackets{\min\braces{1, \frac{\target(\vz) \cdot p_\vz(\vx)}{\target(\vx) \cdot p_\vx (\vz)}}}.
\end{align*}
Applying Markov inequality, we know that
\begin{align*}
    \Exs_{\vz \sim \proposal_\vx}\brackets{\min\braces{1, \dfrac{\target(\vz) \cdot p_\vz(\vx)}{\target(\vx) \cdot p_\vx(\vz)}}} \geq \alpha \Prob\brackets{ \dfrac{\target(\vz) \cdot p_\vz(\vx)}{\target(\vx) \cdot p_\vx(\vz)} \geq \alpha } \text{ for all } \alpha \in (0, 1].
\end{align*}
It is sufficient to derive a high probability lower bound for the ratio $\dfrac{\target(\vz) \cdot p_\vz(\vx)}{\target(\vx) \cdot p_\vx(\vz)}$. Plugging the fact that $\target(\vx) \propto \exp(-\func(\vx))$ and $p_\vx(\vz) \propto \exp\parenth{- \vecnorm{\vx - \step \nabla \func(\vx) - \vz}{2}^2/(4\step)}$, we have
\begin{align*}
    \dfrac{\target(\vz) \cdot p_\vz(\vx)}{ \target(\vx) \cdot p_\vx(\vz)} = \exp\parenth{\frac{4\step \parenth{\func(\vx) - \func(\vz)} + \vecnorm{\vz - \vx + \step\nabla \func(\vx)}{2}^2- \vecnorm{\vx - \vz + \step \nabla \func(\vz)}{2}^2}{4\step}}.
\end{align*}
We then lower bound the term in the numerator of the exponent, without using the convexity of $\func$.
\begin{align*}
    \lefteqn{4 \step \parenth{\func(\vx) - \func(\vz)} + \vecnorm{\vz - \vx + \step \nabla \func(\vx)}{2}^2 - \vecnorm{\vx - \vz + \step \nabla \func(\vz)}{2}^2} \\
    &= 2\step \underbrace{(\func(\vx) - \func(\vz) - \parenth{\vx - \vz}\tp \nabla \func(\vx))}_{M_1} + 2\step \underbrace{(\func(\vx) - \func(\vz) - \parenth{\vx - \vz}\tp \nabla \func(\vz))}_{M_2} + \step^2 \underbrace{\parenth{\vecnorm{\nabla \func(\vx)}{2}^2 - \vecnorm{\nabla \func(\vz)}{2}^2}}_{M_3}.
\end{align*}
Using the fact that $\func$ is smooth, we have
\begin{align*}
    M_1 \geq -\dfrac{\smoothness}{2} \vecnorm{\vx - \vz}{2}^2 \text{ and } M_2 \geq - \frac{\smoothness}{2} \vecnorm{\vx- \vz}{2}^2.
\end{align*}
Again using the smoothness, we have
\begin{align*}
    M_3 = \vecnorm{\nabla \func(\vx)}{2}^2 - \vecnorm{\nabla \func(\vz)}{2}^2 = \angles{\nabla \func(\vx) + \nabla \func(\vz), \nabla \func(\vx) - \nabla \func(\vz)} \geq -\parenth{2 \vecnorm{\nabla \func(\vx)}{2} + \smoothness \vecnorm{\vx- \vz}{2}} \smoothness \vecnorm{\vx-\vz}{2}.
\end{align*}
Combining the bounds $M_1, M_2, M_3$, we have established that
\begin{align*}
    \frac{\target(\vz) \cdot p_\vz(\vx)}{\target(\vx) \cdot p_\vx(\vz)} \geq \exp\underbrace{\parenth{-\frac{1}{2} \smoothness \vecnorm{\vx-\vz}{2}^2 - \frac{\step}{4}\parenth{2\smoothness \vecnorm{\vx-\vz}{2} \vecnorm{\nabla \func(\vx)}{2} + \smoothness^2 \vecnorm{\vx - \vz}{2}^2}}}_{T}.
\end{align*}
In addition, using the fact that $\vz$ is a proposal, we have
\begin{align*}
    \vecnorm{\vx - \vz}{2} = \vecnorm{\step \nabla \func (\vx) + \sqrt{2\step} \xi }{2} \leq \step \vecnorm{f(\vx)}{2} + \sqrt{2\step}\vecnorm{\xi}{2}.
\end{align*}
Simplifying and using the fact that $\smoothness \step \leq 1$, we obtain
\begin{align*}
    T \geq -2 \smoothness \step^2 \vecnorm{\nabla \func (\vx)}{2}^2 - 3 \smoothness \step \vecnorm{\xi}{2}^2 - \smoothness\step \sqrt{\step} \vecnorm{\nabla \func(\vx)}{2} \vecnorm{\xi}{2}.
\end{align*}
Since $\vx \in \truncball$, we can bound the gradient roughly
\begin{align*}
    \vecnorm{\nabla \func(\vx)}{2} = \vecnorm{\nabla \func(\vx) - \nabla \func(\vx^*)}{2} \leq \smoothness \vecnorm{\vx - \vx^*}{2} \leq \smoothness \sqrt{\frac{\dims}{\scparam}} \radius(s) =: \maxgrad_s.
\end{align*}
$\vecnorm{\xi}{2}^2$ is bounded via standard $\chi^2$-variable tail bound. We have
\begin{align*}
    \Prob\brackets{\vecnorm{\xi}{2}^2 \leq \dims \alpha_\epsilon} \geq 1 - \frac{\epsilon}{16},
\end{align*}
for $\alpha_\epsilon = 1+ 2 \sqrt{\log(16/\epsilon)} + 2 \log(16/\epsilon)$. The choice of $\stepbound$ guarantees that for $\step \leq \stepbound$, we have
\begin{align*}
    \smoothness \step^2 \maxgrad_s^2 \leq \frac{\epsilon}{128}, \smoothness\step \dims \alpha_\epsilon \leq \frac{\epsilon}{96}, \text{ and } \smoothness \step \sqrt{\step} \maxgrad_s \sqrt{\dims \alpha_\epsilon} \leq \dfrac{\epsilon}{64}.
\end{align*}
Combining all these bound, we obtain
\begin{align*}
    \Prob\brackets{T \geq -\frac{\epsilon}{16}} \geq 1 - \frac{\epsilon}{16}.
\end{align*}
Using the fact that $e^{-\epsilon/16} \geq 1 - \epsilon / 16$, we have
\begin{align*}
    \Exs_{z\sim \proposal_\vx} \brackets{1, \frac{\target(\vz) \cdot p_\vz(\vx)}{\target(\vx) \cdot p_\vx(\vz)}} \geq 1 - \dfrac{\epsilon}{8}, \text{ for any } \epsilon \in (0, 1), \text{ and } \step \leq \stepbound.
\end{align*}
\end{proof}



\section{Proofs for Optimization}

We denote $\tilde{\nabla} U(\vx) = \{\nabla^n U(\vx) | n \in \mathcal{N}\}$ as shorthand for all $n$-th order derivative at point $\vx$. We consider iterative algorithm class $\mathcal{A}_\infty$ operating on a function $U: \R^d \rightarrow \R$ whose iterates has following form:
\begin{equation*}
\vx_t = g_t(\zeta, \tilde{\nabla} U(\vx_0), \ldots, \tilde{\nabla} U(\vx_{t-1}))
\end{equation*}
where $g_t$ is a mapping to $\R^d$. $\zeta$ is a random variable sampled from uniform distribuion over $[0, 1]$ (indepedent of $U$), and it contains infinitely many random bits.
We note standard optimization algorithms (either deterministic or randomized) which utilize gradient information or any $p$-th order information all fall in to this class of algorithms $\mathcal{A}_\infty$.



\begin{theorem}[Lower bound for optimization] \label{thm:lower_opt}
For any $R>0$, $L \geq 2m > 0$, probability $0 < p \leq 1$, and $\epsilon \leq \dfrac{L R^2}{64\left(2\pi^2+\pi\right)}$,
there exists an objective function $U$ satisfying the local nonconvexity Assumptions~\ref{A1}--\ref{A3} with constants $L$, $m$, and $R$,
such that any algorithm in $\mathcal{A}_\infty$ requires at least $\left\lfloor \left( \dfrac{R}{4}\sqrt{\dfrac{L}{2\pi^2+\pi}}\cdot \dfrac{1}{\sqrt{\epsilon}} - \dfrac12 \right)^d \right\rfloor=\Omega\left( p \cdot \left( LR^2/\epsilon \right)^{d/2}\right)$ iterations to guarantee
$P\left( \min_{\tau \le t}|U(\vx_{\tau}) - U(\vx^*)| < \epsilon \right) \geq p$.
\label{theorem:lower_bound}
\end{theorem}

\subsection{Proof of Theorem~\ref{theorem:lower_bound}}

We constructively prove Theorem~\ref{theorem:lower_bound} by defining such a $U(\vx)$ in what follows.
We first
make use of the following lemma about packing numbers.
Again we denote $\ball(0,R)$ as the closed ball of radius $R$ centered at $0$.
\begin{lemma}[Packing number]
For $R > r > 0$, denote $\eta = \left\lfloor \left( \dfrac{R-r}{2r} \right)^d \right\rfloor$.
Then there exists set $\mathbb{X}_\eta = \{ \vx_1,\cdots \vx_\eta \}$, s.t. $\bigcup_{i=1}^\eta \ball(\vx_i, r) \subset \ball(0,R)$, and $\ball(\vx_i, r) \bigcap \ball(\vx_j, r) = \emptyset, \forall i\neq j$.
\label{lemma:packing_no}
\end{lemma}
%
%
As shown in Fig.~\ref{fig:cartoon_1}, this Lemma~\ref{lemma:packing_no} guarantees the existence of the set $\{\vx_1,\cdots \vx_\eta\}$ so that $\eta$ balls of radius $r$ centered at $\vx_\eta$ are contained inside the larger ball of radius $R$ without intersecting with each other.

We hereby construct $U(\vx)$ that gives the lower bound.
If $\epsilon\geq \dfrac{L R^2}{36(2\pi^2+\pi)}$,
then
\[
T\geq1\geq
p \cdot \left\lfloor \left( \dfrac{R}{4}\sqrt{\dfrac{L}{2\pi^2+\pi}}\cdot \dfrac{1}{\sqrt{\epsilon}} - \dfrac12 \right)^d \right\rfloor, \quad \forall 0<p\leq 1.
\]
Otherwise, take $r = \sqrt{(2\pi^2 + \pi) \epsilon / L}$ in Lemma~\ref{lemma:packing_no}.
Then we have the $r$-packing number inside $\ball(0,R/2)$ to be
\[
\eta = \left\lfloor \left( \dfrac{R/2-r}{2r} \right)^d \right\rfloor
= \left\lfloor \left( \dfrac{R}{4}\sqrt{\dfrac{L}{2\pi^2+\pi}}\cdot \dfrac{1}{\sqrt{\epsilon}} - \dfrac12 \right)^d \right\rfloor \geq 1,
\]
such that there exists set $\mathbb{X}_\eta = \{ \vx_1,\cdots \vx_\eta \}$ satisfying $\bigcup_{i=1}^\eta \ball(\vx_i, r) \subset \ball(0, R)$ and $\forall i\neq j, \ball(\vx_i, r) \bigcap \ball(\vx_j, r) = \emptyset$.
Choose $i^* \in \left\{ 1, \cdots, \eta \right\}$ uniformly at random.
Let
\begin{align}
U(\vx) = \left\{
\begin{array}{l}
\dfrac{L r^2}{4\pi^2+2\pi} \cos\left( \dfrac{\pi}{r^2} \left(||\vx - \vx_{i^*}||_2^2 - r^2\right) \right) - \dfrac{L r^2}{4\pi^2+2\pi}, \quad ||\vx - \vx_{i^*}||_2 \leq r \\
0, \quad ||\vx - \vx_{i^*}||_2 > r, ||\vx||_2 \leq R/2 \\
m \left( ||\vx||_2 - R/2 \right)^2, ||\vx||_2 > R/2.
\end{array}
\right. \label{cost_func}
\end{align}

\begin{lemma}[Lipschitz smoothness and strong convexity]
Let $L\geq 2m$.
Then $U(\vx)$ is $L$-Lipschitz smooth and when $||\vx||_2 > 2R$, $U(\vx)$ is $m$-strongly convex.
\label{smooth_convex}
\end{lemma}

Now we prove that $\forall \; 0 < p \leq 1$, for any algorithm that inputs $\{ U(\vx), \nabla U(\vx), \cdots, \nabla^n U(\vx) \}, \forall n\in\mathcal{N}$, $\forall \; \epsilon < \dfrac{L R^2}{36(2\pi^2+\pi)}$, at least $T \geq p \cdot \eta$ steps are required so that
$P\left( |U(\vx^T) - U(\vx^*)| < \epsilon \right) \geq p$.

Note that for any $\vx^t \not\in \ball(\vx_{i^*}, r)$, $|U(\vx^t) - U(\vx^*)| \geq \dfrac{L r^2}{2\pi^2+\pi} = \epsilon$.
Therefore, probability that $U(\vx^t)$ is $\epsilon$ close to $U(\vx^*)$ is smaller than the probability of $\vx^t \in \ball(\vx_{i^*}, r)$:
\begin{align}
P(|U(\vx^t) - U(\vx^*)|<\epsilon)
&\leq P(\vx^t \in \ball(\vx_{i^*}, r)) \\ \nonumber
&\leq P\left(\vx^t \in \ball(\vx_{i^*}, r) \; \bigg| \; \vx^t \in \bigcup_{j=1}^\eta \ball(\vx_j, r) \right).
\end{align}

We first assume that $\forall t\leq T, \vx^t \in \bigcup_{j=1}^\eta \ball(\vx_j, r) $, then prove that breaking this assumption cannot obtain a better rate of convergence.

\begin{enumerate}
\item
\label{pf:pt_1}
Assume that $\forall t\leq T, \vx^t \in \bigcup_{j=1}^\eta \ball(\vx_j, r)$.
From the definition of $U(\vx)$, \eqref{cost_func}, we know that $\forall j\in \{ 1, \cdots, \eta \}, j\neq {i^*}$, $\forall \vx \in \ball(\vx_j, r)$, $U(\vx) = 0$, $\nabla U(\vx) = 0$, $\cdots$, $\nabla^n U(\vx) = 0$.
%
Hence $\vx^t \in \ball(\vx_j, r)$, $j\neq {i^*}$ only contains information that
${i^*} \in \{ 1,\cdots \eta \}\setminus \{j\}$.
Since $i$ is chosen uniformly at random from $\left\{ 1, \cdots, \eta \right\}$, for $T\leq\eta$
\[
P\left( \vx^T \not\in \ball(\vx_{i^*}, r) \; \Bigg| \; \forall t < T, \vx^t \in \bigcup_{\overset{j=1}{j\neq {i^*}}}^\eta \ball(\vx_j, r)  \right)
\geq \dfrac{\eta - T}{\eta - (T-1)}.
\]
Therefore,
\begin{align}
\lefteqn{P\left(\{\vx^1,\cdots,\vx^T\} \bigcap \ball(\vx_{i^*}, r) = \emptyset \; \bigg| \; \forall t\leq T, \vx^t \in \bigcup_{j=1}^\eta \ball(\vx_j, r)  \right)} \nonumber\\
&\geq \dfrac{\eta - 1}{\eta} \dfrac{\eta - 2}{\eta - 1} \cdots \dfrac{\eta - T}{\eta - (T-1)}
= \dfrac{\eta - T}{\eta}.
\end{align}
This implies: the probability that first passage time into set $\ball(\vx_{i^*}, r)$ is less than or equal to $T$ is:
\begin{align}
\lefteqn{P\left(\{\vx^1,\cdots,\vx^T\} \bigcap \ball(\vx_{i^*}, r) \neq \emptyset \; \bigg| \; \forall t\leq T, \vx^t \in \bigcup_{j=1}^\eta \ball(\vx_j, r) \right)} \nonumber\\
&= 1 - P\left(\{\vx^1,\cdots,\vx^T\} \bigcap \ball(\vx_{i^*}, r) = \emptyset \; \bigg| \; \forall t\leq T, \vx^t \in \bigcup_{j=1}^\eta \ball(\vx_j, r) \right) \nonumber\\
&\leq 1 - \dfrac{\eta - T}{\eta}
= \dfrac{T}{\eta}.
\end{align}
Therefore,
\begin{align}
p &\leq P(|U(\vx^T) - U(\vx^*)|<\epsilon) \nonumber\\
&\leq P\left(\vx^T \in \ball(\vx_{i^*}, r) \; \bigg| \; \vx^T \in \bigcup_{j=1}^\eta \ball(\vx_j, r) \right)  \nonumber\\
&\leq P\left(\{\vx^1,\cdots,\vx^T\} \bigcap \ball(\vx_{i^*}, r) \neq \emptyset \; \Big| \; \forall t\leq T, \vx^t \in \bigcup_{j=1}^\eta \ball(\vx_j, r) \right) \nonumber\\
& \leq \dfrac{T}{\eta},
\end{align}
\[
T \geq p \cdot \eta.
\]

\item
Suppose there exists an algorithm that output $\{\vx_1,\cdots,\vx^T\}$, where $\exists \; t\leq T, \; \vx^t \not\in \bigcup_{j=1}^\eta \ball(\vx_j, r)$ and finds $\vx_{i^*} + r \ball$ with probability $p$ within less than $p\cdot\eta$ steps.
Then design a corresponding algorithm that outputs $\{\vx_1,\cdots,\vx^T\} {\setminus} \left\{ \vx \big| \vx \not\in \bigcup_{j=1}^\eta \ball(\vx_j, r) \right\}$ so that $\forall \; t\leq T, \; \vx^t \in \bigcup_{j=1}^\eta \ball(\vx_j, r)$, and $\ball(\vx_{i^*}, r)$ is found with probability $p$ within less than $p\cdot\eta$ steps.
But this contradicts with \ref{pf:pt_1}.
\end{enumerate}

\subsubsection{Supporting Proofs for Theorem~\ref{theorem:lower_bound}}
\begin{proof}[Proof of Lemma~\ref{lemma:packing_no} (Packing number)]
Let $\mathcal{P}(r, \ball(0, R), || \cdot ||_2)$ be the $r$-packing number of $\ball(0, R)$; and $\mathcal{C}(r, \ball(0, R), || \cdot ||_2)$ be the $r$-covering number of $\ball(0, R)$.
One can follow the properties of packing and covering numbers to proved that: $\mathcal{P}(r, \ball(0, R), || \cdot ||_2) \geq \mathcal{C}(r, \ball(0, R), || \cdot ||_2) \geq \left\lfloor \left( \dfrac{R}{r} \right)^d \right\rfloor$.
Therefore, number of non-intersecting $r$-balls that can be contained in an $\ball(0, R)$ is $\mathcal{P}(2r, \ball(0, R-r), || \cdot ||_2) \geq \left\lfloor \left( \dfrac{R-r}{2r} \right)^d \right\rfloor$.
\end{proof}

\begin{proof}[Proof of Lemma~\ref{smooth_convex} (Lipschitz smoothness and strong convexity)]
We first prove that when $||\vx||_2 \leq R/2$, $U(\vx)$ is $L$-Lipschitz smooth.
We then prove that when $||\vx||_2 > R/2$, $U(\vx)$ is $2m$-Lipschitz smooth.
At last we prove that $U(\vx)$ is $m$-strongly convex for $||\vx||_2 > R$.
Since $L\geq 2m$, this proves Lemma~\ref{smooth_convex}.
\begin{itemize}
\item
Define $U_1(\vx) = \cos\left( \dfrac{\pi}{r^2} \left(||\vx - \vx_i||_2^2 - r^2\right) \right)$.
Then $U(\vx) = \dfrac{L r^2}{4\pi^2+2\pi} \left(U_1(\vx) - 1\right)$ when $||\vx - \vx_i||_2 \leq r$.

Hessian of $U_1$ is:
\begin{align}
H[U_1](\vx) &= - \dfrac{4\pi^2}{r^4} \cos\left( \dfrac{\pi}{r^2} \left(||\vx - \vx_i||_2^2 - r^2\right) \right) (\vx - \vx_i) (\vx - \vx_i)^\rT \nonumber\\
&- \dfrac{2\pi}{r^2} \sin\left( \dfrac{\pi}{r^2} \left(||\vx - \vx_i||_2^2 - r^2\right) \right) \mI. \nonumber
\end{align}
We first note that $\left|\left|(\vx - \vx_i) (\vx - \vx_i)^\rT\right|\right|_2 = || \vx - \vx_i ||_2^2 \leq r^2$.
Hence,
\begin{align}
|| H[U_1](\vx) ||_2 &\leq
\left|\left| \dfrac{4\pi^2}{r^4} \cos\left( \dfrac{\pi}{r^2} \left(||\vx - \vx_i||_2^2 - r^2\right) \right) (\vx - \vx_i) (\vx - \vx_i)^\rT \right|\right|_2 \nonumber\\
&+ \left|\left| \dfrac{2\pi}{r^2} \sin\left( \dfrac{\pi}{r^2} \left(||\vx - \vx_i||_2^2 - r^2\right) \right) \mI \right|\right|_2 \nonumber\\
&= \dfrac{4\pi^2+2\pi}{r^2}. \nonumber
\end{align}
Therefore, when $||\vx - \vx_i||_2 \leq r$, $U(\vx) = \dfrac{L r^2}{4\pi^2+2\pi} \left(U_1(\vx) - 1\right)$ is $L$-Lipschitz smooth.

When $||\vx - \vx_i||_2 > r$ and $||\vx||_2 \leq R$, $U(\vx)=0$ is also $L$-Lipschitz smooth, which leads to the result that $U(\vx)$ is $L$-Lipschitz smooth for $||\vx||_2 \leq R$.
\item
Define $U_2(\vx) = \left( ||\vx||_2 - R/2 \right)^2$.
Then $U(\vx) = m U_2(\vx)$ when $||\vx||_2 > R/2$.

\[
H[U_2](\vx) = 2 \left(1 - \dfrac{R}{2||\vx||_2}\right) \mI + \dfrac{R}{||\vx||_2^3} \vx \vx^\rT.
\]
Similar to above, it can be proven that $||\vx \vx^\rT||_2 = ||\vx||_2^2$.
Hence $||H[U_2](\vx)||_2 \leq 2\left| 1 - \dfrac{R}{2||\vx||_2} \right| + \dfrac{R}{||\vx||_2} = 2$.
Therefore, $m U_2(\vx)$ is $2m$-Lipschitz smooth for $||\vx||_2 > R/2$.


\item
Define
\[
U_3(\vx) = \left\{
\begin{array}{l}
U_2(\vx), \ \lrn{\vx}_2 > R \\
\dfrac{1}{2} \lrn{\vx}_2^2 + \dfrac{1}{8} R^2, \ \lrn{\vx}_2 \leq R
\end{array}
\right..
\]
Then
\[
H\left[U_3(\vx) - \dfrac{1}{2} \lrn{\vx}_2^2\right] =
\left\{
\begin{array}{l}
\left(1 - \dfrac{R}{||\vx||_2}\right) \mI + \dfrac{R}{||\vx||_2^3} \vx \vx^\rT, \ \lrn{\vx}_2 > R \\
0, \ \lrn{\vx}_2 \leq R
\end{array}
\right..
\]
For any $\vy\in\mathbb{R}^d$, $\vy^\rT \vx \vx^\rT \vy = (\vy^\rT \vx)^2 \geq 0$.
Therefore all eigenvalues of $\vx \vx^\rT$ are bigger than or equal to $0$.
Since $\mI$ can be simultaneously diagonalized with $\vx \vx^\rT$, $H\left[U_3(\vx) - \dfrac{1}{2} \lrn{x}_2^2\right] \succeq \left(1 - \dfrac{R}{||\vx||_2}\right) \mI \succeq 0$ when $||\vx||_2 > R$.
When $||\vx||_2 \leq R$, $H\left[U_3(\vx) - \dfrac{1}{2} \lrn{\vx}_2^2\right]=0$.
Also note that $U_3(\vx) - \dfrac{1}{2} \lrn{\vx}_2^2$ is continuously differentiable.
Hence $U_3(\vx) - \dfrac{1}{2} \lrn{\vx}_2^2$ is convex.

On the other hand, $U(\vx) = m U_3(\vx)$ when $||\vx||_2 > R$.
Following Assumption~\ref{A2}, this implies that $U(\vx) - \dfrac{m}{2} \lrn{\vx}_2^2$ is convex on $\R^d\setminus\ball(0,R)$.
Therefore, $U(\vx)$ is $m$-strongly convex on $\R^d\setminus\ball(0,R)$.
\end{itemize}
\end{proof}

\subsection{Proof of Corollary~\ref{lemma:relation}}
\begin{corollary}
\label{lemma:relation}
There exists an objective function $U$ that is
$m$-strongly convex outside of a region of radius $R$ and $L$-Lipschitz smooth,
such that for $\hat{\vx} \sim q^{*}_{\beta}$, it is required that $\beta=\widetilde{\Omega}\left(d/\epsilon\right)$ to have
$U(\hat{\vx})-U\left(\vx^*\right)<\epsilon$ for a constant probability.
Moreover, number of iterations required for the Langevin algorithms
is $K = e^{\widetilde{\mathcal{O}} \left( d \cdot LR^2/\epsilon \right)}$ to guarantee that $U(\vx^K)-U\left(\vx^*\right)<\epsilon$ for a constant probability.
\end{corollary}

To use Langevin algorithm to attain optimal value with probability $p$, we separate the optimization problem into two:
one is to find a parameter $\beta$ such that $\hat{\vx} \sim q^*_\beta \propto e^{-\beta U}$ has probability $p$ of being close to the optimum $\vx^*$ (i.e., $P(U(\hat{\vx})-U\left(\vx^*\right)<\epsilon)\geq p$);
another is to sample from a distribution $q^{K}_{\beta}$ after $K$-th iteration so that it is $\delta$-close to $q^*_\beta$, for $\delta \leq p/2$ in TV distance.
Then by the definition of TV distance, $\vx^K \sim q^{K}_{\beta}$ will have probability $p/2$ of being close to the optimum $\vx^*$.

\begin{proof}[Proof of Corollary~\ref{lemma:relation}]
We take $U$ as the one defined in \eqref{cost_func} and similarly take $r=\sqrt{(2\pi^2+\pi)\epsilon/L}$.
Then $\vx^* = \arg\min_{\vx\in\mathbb{R}^d} U(\vx) = \vx_{i^*}$ and $\min_{\vx\in\mathbb{R}^d} U(\vx) = -\dfrac{L r^2}{2\pi^2+\pi} = -\epsilon$.
For $U(\hat{\vx})-U\left(\vx^*\right)<\epsilon$, it is required that $\lrn{\hat{\vx}-\vx^*} \leq r$.

If $\hat{\vx}$ follows the law of $q^{*}_{\beta}$, then denote the associated probability measure $\rd \Pi^*_{\beta} = q^{*}_{\beta} \rd \hat{\vx}$.
We then estimate the probability that $\hat{\vx}\in\ball\left(\vx^*, r\right)$
\begin{align}
P\left(\lrn{\hat{\vx}-\vx^*} \leq r\right)
&= \Pi^*_{\beta} \big( \ball\left(\vx^*, r\right) \big)
\nonumber\\
&= \dfrac{ \int_{\ball\left(\vx^*, r\right)} e^{-\beta U(\vx)} \rd \vx }
{ \int_{\ball\left(\vx^*, r\right)} e^{-\beta U(\vx)} \rd \vx
+ \int_{\ball\left(0, R/2\right) \setminus\ball\left(\vx^*, r\right) } e^{-\beta U(\vx)} \rd \vx
+ \int_{\mathbb{R/2}^d\setminus\ball\left(0, R/2\right)} e^{-\beta U(\vx)} \rd \vx }
\nonumber\\
&= \dfrac{ \int_{\ball\left(\vx^*, r\right)} e^{-\beta U(\vx)} \rd \vx }
{ \int_{\ball\left(\vx^*, r\right)} e^{-\beta U(\vx)} \rd \vx
+ \int_{\ball\left(0, R/2\right) \setminus\ball\left(\vx^*, r\right) } 1 \ \rd \vx
+ \int_{\mathbb{R/2}^d\setminus\ball\left(0, R/2\right)} e^{-\beta U(\vx)} \rd \vx }
\nonumber\\
&= \dfrac{ \int_{\ball\left(\vx^*, r\right)} e^{-\beta U(\vx)} \rd \vx }
{ \int_{\ball\left(\vx^*, r\right)} \left( e^{-\beta U(\vx)} - 1 \right) \rd \vx
+ \int_{\ball\left(0, R/2\right)} 1 \ \rd \vx
+ \int_{\mathbb{R/2}^d\setminus\ball\left(0, R/2\right)} e^{-\beta U(\vx)} \rd \vx }
\nonumber\\
&\leq \dfrac{ \int_{\ball\left(\vx^*, r\right)} e^{-\beta U(\vx)} \rd \vx }
{ \int_{\ball\left(0, R/2\right)} 1 \ \rd \vx }
\nonumber\\
&\leq \dfrac{ e^{-\min_{\lrn{\vx-\vx^*} \leq r} \beta U(\vx)} \int_{\ball\left(\vx^*, r\right)} 1 \ \rd \vx }
{ \int_{\ball\left(0, R/2\right)} 1 \ \rd \vx }
\nonumber\\
&= e^{\beta \epsilon} \dfrac{ \int_{\ball\left(\vx^*, r\right)} 1 \ \rd \vx }
{ \int_{\ball\left(0, R/2\right)} 1 \ \rd \vx }
\nonumber\\
&= e^{\beta \epsilon} \left(\dfrac{2r}{R}\right)^d.
\end{align}
To obtain that $P(U(\hat{\vx})-U\left(\vx^*\right)<\epsilon)=P\left(\lrn{\hat{\vx}-\vx^*} \leq r\right)\geq p$,
we need that
\[
e^{\beta \epsilon} \left(\dfrac{2r}{R}\right)^d \geq p.
\]
Therefore,
\[
\beta
\geq \dfrac{1}{\epsilon} \ln p + \dfrac{d}{\epsilon} \ln\left(\dfrac{R}{2r}\right)
= \dfrac{1}{\epsilon} \ln p + \dfrac{1}{2} \dfrac{d}{\epsilon} \ln\left(\dfrac{1}{4(2\pi^2+\pi)} \dfrac{LR^2}{\epsilon}\right).
\]

To use the Langevin algorithms to search for optimum, we are actually using $\vx^K$, which follows the sampled distribution $q^{K}_{\beta}$ at $K$-th step.
And we are taking $K$ large enough so that $\lrn{ q^{K}_{\beta} - q^{*}_{\beta} }_{{\text TV}} \leq \delta$, for $\delta \leq p/2$.
Then, for a large enough $K$, we can have
\begin{align}
\lefteqn{\left| P\left(\lrn{\vx^K-\vx^*} \leq r\right) - P\left(\lrn{\hat{\vx}-\vx^*} \leq r\right) \right|}
\nonumber\\
&= \left| \Pi^K_{\beta} \big( \ball\left(\vx^*, r\right) \big) - \Pi^*_{\beta} \big( \ball\left(\vx^*, r\right) \big) \right|
\nonumber\\
&\leq \sup_{A} \left| \Pi^K_{\beta}(A) - \Pi^*_{\beta}(A) \right|
\nonumber\\
&= \lrn{ q^{K}_{\beta} - q^{*}_{\beta} }_{{\text TV}}
\nonumber\\
&\leq \delta,
\end{align}
which guarantees that $P\left(\lrn{\vx^K-\vx^*} \leq r\right) \geq p/2$.

We directly obtain from Theorem~\ref{thm:MCMC} that for the objective function $\beta U$ with Lipschitz constant $\beta L \geq \dfrac{L}{\epsilon} \ln p + \dfrac{d}{2} \dfrac{L}{\epsilon} \ln\left(\dfrac{1}{4(2\pi^2+\pi)} \dfrac{LR^2}{\epsilon}\right)$, we need to iterate $e^{\widetilde{\mathcal{O}} \left( d \cdot LR^2/\epsilon \right)}$ steps to guarantee convergence.

\end{proof}

\section{Proofs for Gaussian Mixture Models}

\label{sec:GMM_proof}

Consider the problem of inferring mean parameters $\pmu=(\mu_1,\cdots,\mu_M)\in\mathbb{R}^{d\times M}$ in a Gaussian mixture model with $M$ mixtures from $N$ data $\vy=(y_1,\cdots,y_N)$:
\begin{align}
&p\left(y_n | \pmu\right) = \sum_{i=1}^M \dfrac{\lambda_i}{Z_i} \exp\left( - \dfrac{1}{2} (y_n - \mu_i)^\rT \Sigma_i^{-1} (y_n - \mu_i) \right)
+ \left( 1 - \sum_{i=1}^M \lambda_i \right) p_0(y_n),
\label{eq:GMM}
\end{align}
where
$Z_i$ are the normalization constants and $\sum_{i=1}^M \lambda_i \leq 1$.
For succinctness, we consider in this section the cases where covariances $\Sigma_i$ are isotropic and uniform across all mixture components: $\Sigma_i = \Sigma = \sigma^2 \mI$.
 $p_0(y_n)$ represents crude observations of the data (e.g., data may be distributed inside a region or may have sub-Gaussian tail behavior).
The objective function is given by the log posterior distribution: $U(\pmu) = -\log p(\pmu) - \sum_{n=1}^N \log p\left(y_n | \pmu\right)$. Assume data are distributed in a bounded region ($\lrn{y_n}\leq R$) and take $p_0(y_n)= \ind{\|y_n\|\leq R} / Z_0$ to describe this observation.

We also take the prior to be
\begin{align}
p(\pmu) \propto \exp\left( - m \left(\|\pmu\|_F - {\sqrt{M}} R\right)^2 \ind{\|\pmu\|_F \geq {\sqrt{M}} R} \right).
\label{eq:GMM_prior}
\end{align}

\subsection{Proofs for Smoothness}
\begin{fact}
For the Gaussian mixture model defined in~\eqref{eq:GMM},
define
\begin{align}
\alpha = \dfrac{1}{\sigma^2}\max\left\{ 2\sup_{\mu\in\{\mu_1,\cdots,\mu_M\}} \sum_{n=1}^N \dfrac{\lrn{\mu-y_n}^2}{\sigma^2} \exp\left( - || \mu - y_n ||^2\big/2\sigma^2 \right) ,  \sup_{\mu\in\{\mu_1,\cdots,\mu_M\}} \sum_{n=1}^N \exp\left( - || \mu - y_n ||^2\big/2\sigma^2 \right) \right\}. \label{eq:lambda_i_accurate_1}
\end{align}
If we take
$\lambda_i = \dfrac{\dfrac{l}{\alpha} Z_i}{ Z_0 + \dfrac{l}{\alpha} \sum_{j=1}^M Z_j}$,
then the log-likelihood $- \sum_{n=1}^N \log p\left(y_n | \pmu\right)$ is $l$-Lipschitz smooth.
\label{fact:GMM_smoothness}
\end{fact}

\begin{proof}[Proof of Fact~\ref{fact:GMM_smoothness}]
Define the mixture components: $W_{i,n}=\dfrac{\lambda_i}{Z_i} \exp\left( - \dfrac{1}{2}|| y_n - \mu_i ||^2/\sigma^2 \right)$ and $C_n=\left( 1 - \sum_{i=1}^M \lambda_i \right) p_0(y_n)$.
Since all the data $\{y_n\}$ are distributed in $\ball(0,R)$, $p_0(y_n)=\dfrac{1}{Z_0} \ind{\|y_n\|\leq R}=\dfrac{1}{Z_0}$.
We can plug in the expression of $\lambda_i = \dfrac{\dfrac{l}{\alpha} Z_i}{ Z_0 + \dfrac{l}{\alpha} \sum_{j=1}^M Z_j}$ and obtain for any $n=1,\cdots,N$:
\begin{align}
C_n = C &= \dfrac{1}{Z_0} \left( 1 - \sum_{i=1}^M \lambda_i \right) \nonumber\\
&= \dfrac{1}{Z_0} \left( 1 - \dfrac{ \dfrac{l}{\alpha} \sum_{i=1}^M Z_i }{ Z_0 + \dfrac{l}{\alpha} \sum_{j=1}^M Z_j} \right) \nonumber\\
&= \dfrac{1}{ Z_0 + \dfrac{l}{\alpha} \sum_{j=1}^M Z_j}. \label{eq:def_C_simplify}
\end{align}
Then we can use $C$ to simplify the expression of $\lambda_i$ for $i=1,\cdots,M$:
\[
\lambda_i = \dfrac{l}{\alpha} C Z_i.
\]

We also represent the objective function as:
\[
U(\pmu)=-\log p(\pmu) - \sum_{n=1}^N \log p\left(y_n | \pmu\right)
=-\log p(\pmu) - \sum_{n=1}^N \log\left(\sum_{i=1}^M W_{i,n} + C\right),
\]
and define
\[
\gamma_{i,n} = \dfrac{W_{i,n}}{\sum_{k=1}^M W_{k,n} + C}.
\]
One can find that
\[
- \nabla_{\mu_i} \log p\left(y_n | \pmu\right) = \dfrac{W_{j,n}}{\sum_{j=1}^M W_{j,n} + C} \dfrac{\mu_i-y_n}{\sigma^2}
= \gamma_{i,n} \dfrac{\mu_i-y_n}{\sigma^2},
\]
and
\begin{align*}
- \nabla^2_{\mu_i, \mu_j} \log p\left(y_n | \pmu\right) &=
\left\{
\begin{array}{l}
\dfrac{\gamma_{i,n}}{\sigma^2} \mI + (\gamma_{i,n}^2 - \gamma_{i,n}) \dfrac{(\mu_i-y_n)(\mu_i-y_n)^T}{\sigma^4}, \quad i=j \\
\gamma_{i,n}\gamma_{j,n} \dfrac{(\mu_i-y_n)(\mu_j-y_n)^T}{\sigma^4}, \quad i\neq j
\end{array}
\right. .
\end{align*}
For any vector $\vv$,
\begin{align*}
\lefteqn{- \vv^T \nabla^2_{\pmu^2} \log p\left(y_n | \pmu\right) \vv}
\\
&= \sum_{i=1}^M \dfrac{\gamma_{i,n}}{\sigma^2} v_i^T v_i
- \sum_{i=1}^M \gamma_{i,n} \left[v_i^T\left(\dfrac{\mu_i-y_n}{\sigma^2}\right)\right]^2 \\
&+ \sum_{i=1}^M \sum_{j=1}^M \gamma_{i,n} \gamma_{j,n} \left[v_i^T\left(\dfrac{\mu_i-y_n}{\sigma^2}\right)\right] \left[v_j^T\left(\dfrac{\mu_j-y_n}{\sigma^2}\right)\right].
\end{align*}
Since $\sum_{i=1}^M \gamma_{i,n} = \sum_{i=1}^M \dfrac{W_{i,n}}{\sum_{k=1}^M W_{k,n} + C} \leq 1$,
\begin{align*}
\lefteqn{\left| \sum_{i=1}^M \sum_{j=1}^M \gamma_{i,n} \gamma_{j,n} \left[v_i^T\left(\dfrac{\mu_i-y_n}{\sigma^2}\right)\right] \left[v_j^T\left(\dfrac{\mu_j-y_n}{\sigma^2}\right)\right] \right|}
\\ &\leq
\dfrac{1}{2} \sum_{i=1}^M \sum_{j=1}^M \gamma_{i,n} \gamma_{j,n} \left(\left[v_i^T\left(\dfrac{\mu_i-y_n}{\sigma^2}\right)\right]^2 + \left[v_j^T\left(\dfrac{\mu_j-y_n}{\sigma^2}\right)\right]^2\right)
\\ &\leq
\gamma_{i,n} \left[v_i^T\left(\dfrac{\mu_i-y_n}{\sigma^2}\right)\right]^2 .
\end{align*}
Therefore,
\begin{align*}
{\rm diag}\left( \dfrac{\gamma_{i,n}}{\sigma^2}\left(1-2\dfrac{\|\mu_i-y_n\|^2}{\sigma^2} \right) \mI \right)
\preceq \nabla^2_{\pmu^2} \log p\left(y_n | \pmu\right) \preceq
{\rm diag} \left( \dfrac{\gamma_{i,n}}{\sigma^2} \mI \right).
\end{align*}

Since $\{W_{i,n}\}$ are positive,
\[
\gamma_{i,n} = \dfrac{W_{i,n}}{\sum_{j=1}^M W_{j,n} + C}
\leq \dfrac{W_{i,n}}{C}
= \dfrac{\lambda_i}{C Z_i} \exp\left( - || \mu_i - y_n ||^2\big/2\sigma^2 \right).
\]
Since
\begin{align}
\alpha = \dfrac{1}{\sigma^2}\max\left\{ 2\sup_{\mu} \sum_{n=1}^N \dfrac{\lrn{\mu-y_n}^2}{\sigma^2} \exp\left( - || \mu - y_n ||^2\big/2\sigma^2 \right) ,  \sup_{\mu} \sum_{n=1}^N \exp\left( - || \mu - y_n ||^2\big/2\sigma^2 \right) \right\}. \label{eq:lambda_i_accurate_2}
\end{align}
and
\begin{align}
\lambda_i = \dfrac{l}{\alpha} C Z_i, \label{eq:lambda_i}
\end{align}
log-likelihood $- \sum_{n=1}^N \log p\left(y_n | \pmu\right)$ is $l$-Lipschitz smooth.
It can be seen from~\eqref{eq:lambda_i_accurate_2} that if one uses a loose upper bound for $\alpha$, we can simply take $\lambda_i$ to be $\dfrac{l}{2} \dfrac{C Z_i \sigma^2}{N}$.
\end{proof}

\subsection{Proofs for the EM Algorithm}
\label{sec:dataset}

\begin{figure}
\centering
\includegraphics[scale=0.4]{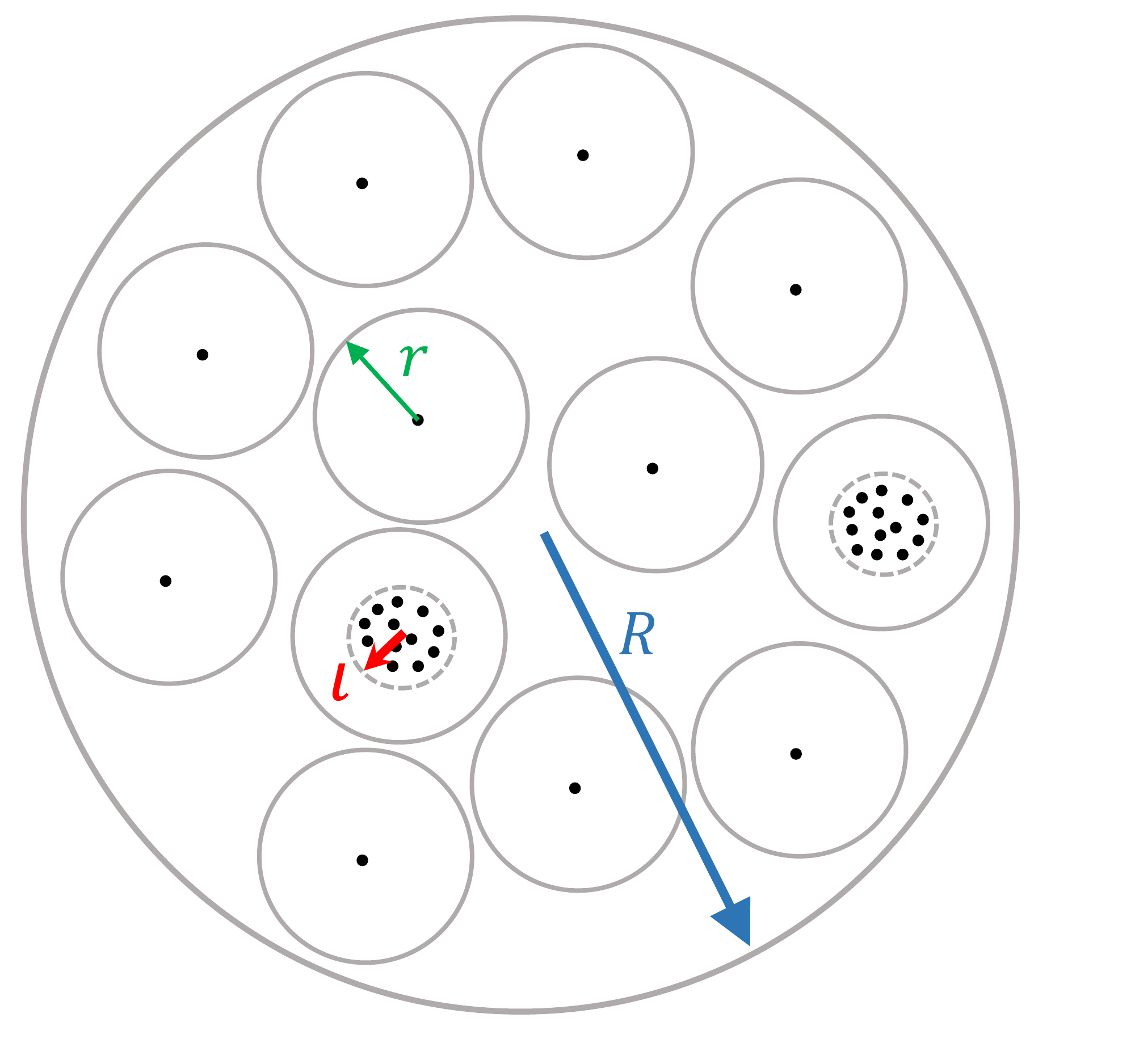}
\caption{Cartoon of an exemplified dataset.}
\label{fig:cartoon_2}
\end{figure}

We prove in the following Lemma that there exists a dataset $(y_1,\cdots,y_N)$ and variance $\sigma^2$ with the previous setting that takes $K\geq\min\{\mathcal{O}(d^{1/\epsilon}), \mathcal{O}(d^d)\}$ steps for the EM algorithm to converge if one initializes the algorithm close to the given data points.
\begin{lemma}
Let the objective function $U(\pmu) = -\log p(\pmu) - \sum_{n=1}^N \log p\left(y_n | \pmu\right)$ with prior $p(\pmu)$ and likelihood $p\left(y_n | \pmu\right)$ defined in \eqref{eq:GMM_prior} and \eqref{eq:GMM}.
Take the parameters $\lambda_i$ so that the log-likelihood is Lipschitz smooth with Lipschitz constant
$L=1/16$,
strong convexity constant $m=1/64$ outside of region with radius $R=1/2$, and number of mixtures $M=\log_2 d$.
Then there exists a dataset $(y_1,\cdots,y_N)$ and variance $\sigma^2$ so that the EM algorithm will take $K\geq\min\{\mathcal{O}(d^{1/\epsilon}), \mathcal{O}(d^d)\}$ queries to converge to $\mathcal{O}(\epsilon)$ close to the optimum if one randomly initializes the algorithm $0.01$ close to the given data points.
\label{lemma:EM_converge}
\end{lemma}
Proof of Lemma~\ref{lemma:EM_converge} shares similar traits as that in~\citep{Chi_GMM,Mathias_GMM}.

Directly invoking Theorem~\ref{thm:MCMC}, we know that the Langevin algorithms converge within $K \leq \widetilde{\mathcal{O}}\left(d^3/\epsilon\right)$ and
$K \leq \widetilde{\mathcal{O}}\left(d^3 \ln^2\left(1/\epsilon\right)\right)$
steps, respectively.

\begin{proof}[Proof of Lemma~\ref{lemma:EM_converge}]
Consider a dataset with $N$ number of $d$-dimensional data points, $y_n\in\mathbb{R}^d$, $n=1,\cdots,N$, described below.
We suppose that it is modeled with $M<N$ mixture components
in the Gaussian mixture model~\eqref{eq:GMM}.

For the first $N - 9M$ points, let $\|y_n\|\leq 0.45$, and $\|y_k-y_l\|\geq 0.11$, where $n,k,l\in\{1, \cdots, N - 9M\}$ and $k\neq l$.
From Lemma~\ref{lemma:packing_no}, we know that when $N\leq2^{d}$, this setting is feasible.
For the next $9M$ points, first select $M$ different indices $\{i_1,\cdots,i_{M}\}$ from $\{1, \cdots, N - 9M\}$ uniformly at random.
Then for $n\in\{ N - 9M + 9(k-1) + 1 , \cdots, N - 9M + 9k \}$ ($k\in\{1,\cdots,{M}\}$), $\|y_n-y_{i_k}\|\leq \sigma/2$.

By this setting, $\forall y_n$, $\|y_n\|\leq 0.5$.
Furthermore, when $n,\hat{n}\in\{ N - 9M + 9(k-1) + 1 , \cdots, N - 9M + 9k \}\cup\{i_k\}$, $\|y_{n}-y_{\hat{n}}\| \leq \sigma/2$;
otherwise, $\|y_{n}-y_{\hat{n}}\| \geq 0.1$ for $n\neq\hat{n}$.
We depict a cartoon of this dataset in Fig.~\ref{fig:cartoon_2}.

Since it can be observed that all the data are distributed in $\ball(0,0.5)$, we let $p_0(y_n) = \dfrac{1}{Z_0} \ind{\|y_n\|\leq 0.5} = \dfrac{\Gamma(d/2+1)}{(2\pi)^{d/2}} \ind{\|y_n\|\leq 0.5}$.
Inclusion of $p_0$ provides a better description of the data, since they are mostly distributed uniformly in $\ball(0,0.5)$, with some concentrated around the chosen $M$ centers.
Then according to \eqref{eq:GMM_prior}, we set the prior to be:
\begin{align*}
p(\pmu) \propto \exp\left( - \dfrac{\left(\|\pmu\|_F - {\sqrt{M}}/{2}\right)^2}{64} \ind{\|\pmu\|_F \geq {\sqrt{M}}/{2}} \right),
\end{align*}
where $\|\pmu\|_F = \sqrt{\sum_{i=1}^M \|\mu_i\|_2^2}$.
Note that in this setting, the positions of local minima are exactly the same as the Gaussian mixture model that does not include prior observation $p_0(y)$ and prior belief $p(\pmu)$.

We take $\lambda_i=\dfrac{1}{64\alpha} C Z_i$ (using notations defined in~\eqref{eq:lambda_i} and~\eqref{eq:def_C_simplify}).
Then the objective function defined via the log posterior:
\begin{align}
U(\pmu)&= -\log p(\pmu) - \sum_{n=1}^N \log p\left(y_n | \pmu\right)
\nonumber\\
&= -\log p(\pmu)
- \sum_{n=1}^N \log\left(\sum_{i=1}^M \dfrac{\lambda_i}{Z_i} \exp\left( - \dfrac{1}{2}|| y_n - \mu_i ||^2/\sigma^2 \right) + C \right)
\nonumber\\
&= \dfrac{\left(\|\pmu\|_F - {\sqrt{M}}/{2}\right)^2}{64} \ind{\|\pmu\|_F \geq {\sqrt{M}}/{2}} \nonumber\\
& - \sum_{n=1}^N \log\left(\sum_{i=1}^M \dfrac{1}{64\alpha} \exp\left( - \dfrac{1}{2}|| y_n - \mu_i ||^2/\sigma^2 \right) + 1\right) + \widetilde{C} \label{eq:potential_specific}
\end{align}
has Lipschitz smoothness $L\leq1/32$.
In what follows, we take $\sigma=\sigma=\dfrac{0.01}{\sqrt{\log_2 N}}$.

It can be seen that $\alpha$ in \eqref{eq:lambda_i_accurate_1} is bounded as: $\alpha \leq \dfrac{50}{\sigma^2}$.
Then $\lambda_i=\dfrac{1}{3200} C Z_i \sigma^2$. 
It can also be checked that the objective function $U$ is also $m\geq1/64$ strongly convex for $\|\pmu\|_F \geq \sqrt{M}$.


We then estimate number of fixed points for ${\|\pmu\|_F\leq{\sqrt{M}}/{2}}$ when running the EM algorithm.
If we run the EM algorithm starting with ${\|\pmu^{(t)}\|_F\leq{\sqrt{M}}/{2}}$, we first compute the weights for each component using old value $\pmu^{(t)}$ (in E step):
\begin{align}
\gamma_{i,n}^{(t)} &= \dfrac{\dfrac{\lambda_i}{Z_i} \exp\left( - \dfrac{1}{2}|| y_n - \mu_i^{(t)} ||^2/\sigma^2 \right)}{\sum_{j=1}^M \dfrac{\lambda_j}{Z_j} \exp\left( - \dfrac{1}{2}|| y_n - \mu_j^{(t)} ||^2/\sigma^2 \right) + C} \nonumber\\
&= \dfrac{\dfrac{\sigma^2}{3200} \exp\left( - \dfrac{1}{2}|| y_n - \mu_i^{(t)} ||^2/\sigma^2 \right)}{\sum_{j=1}^M \dfrac{\sigma^2}{3200} \exp\left( - \dfrac{1}{2}|| y_n - \mu_j^{(t)} ||^2/\sigma^2 \right) + 1}. \label{eq:gamma_expression}
\end{align}
We then update $\pmu$ (in M step):
\begin{align*}
\mu_i^{(t+1)} = \sum_{n=1}^N \dfrac{\gamma_{i,n}^{(t)}}{ \sum_{\hat{n}=1}^N \gamma_{i,\hat{n}}^{(t)} } y_n.
\end{align*}

We prove in Lemma~\ref{lemma:EM_fix_point} that $\forall y_{n_i}, n_i\in\{1,\cdots,N/2\}$, if $\|\mu_i^{(0)} - y_{n_i}\| \leq 0.01$, then $\|\mu_i^{(\tau)} - y_{n_i}\| \leq 0.01$, $\forall \tau>0$.
Therefore, any $M$ combinations of $N-9M$ data points is a fixed point for $\pmu$.

\begin{lemma}
Suppose we run the EM algorithm with the dataset specified in the beginning of Sec.~\ref{sec:dataset} for $T$ steps.
If we initialized each component of $\pmu$ with $\|\mu_i^{(0)} - y_{n_i}\| \leq {0.01}$ for ${n_i}\in\{1,\cdots,N/2\}$, then $\|\mu_i^{(\tau)} - y_{n_i}\| \leq {0.01}$, $\forall \tau>0$. \label{lemma:EM_fix_point}
\end{lemma}

We note that the global minima $\pmu^*=(\mu_1^*,\cdots,\mu_M^*)$ will have $\forall i\in\{1,\cdots,M\}$, $\mu_i^*\in\bigcup_{k=1}^M \Omega_k$, where we denote $\Omega_k = \{ N - 9M + 9(k-1) + 1 , \cdots, N - 9M + 9k \}\cup\{i_k\}$.
It can also be checked from~\eqref{eq:potential_specific} that the difference $\epsilon$ between the global minima and any local minimum $\bar{\pmu}$ that has $\exists i\in\{1,\cdots,M\}$, s.t. $\bar{\mu}_i\notin\bigcup_{k=1}^M \Omega_k$ scales with $N$ as $\epsilon = \mathcal{O}(\sigma^2) = \mathcal{O}\left(\dfrac{1}{\log_2 N}\right)$.
Therefore, if one randomly initialize from the dataset, to attain
global minima
with probability $p$, at least
$K = p \cdot
\left(
\begin{array}{c}
N \\
M
\end{array}
\right)
\bigg/
\left(
\begin{array}{c}
10M \\
M
\end{array}
\right)
\geq
p \cdot \left(\dfrac{N}{10M}\right)^M$
re-initializations are required.
Let $N\gg M$.
Then the number of re-initializations are of order $K=\mathcal{O}(p \cdot N^M)$.

Note that we have taken $M=\log_2 d$.
For $\epsilon>\mathcal{O}(1/d)$, take $N=\mathcal{O}\left(2^{1/\epsilon}\right)$.
Then $T=\mathcal{O}\left(d^{1/\epsilon}\right)$.
For $\epsilon\leq\mathcal{O}(1/d)$, take $N=2^{d}$.
Then $T=\mathcal{O}(d^{d})$.
So $T=\min\left\{\mathcal{O}\left(d^{1/\epsilon}\right),\mathcal{O}(d^{d})\right\}$.
\end{proof}

\begin{remark}
It can be similarly proven that the gradient descent algorithm with its stepsize tuned according to the Lipschitz smoothness has the same behavior if initialized randomly from the dataset.
\end{remark}

\begin{proof}[Proof of Lemma~\ref{lemma:EM_fix_point}]
We prove for each component $\mu_i$ using induction over $t\in\{0,\cdots,\tau\}$.
First assume that $\|\mu_i^{(t)} - y_{n_i}\| \leq {0.01}$.

Then we observe from~\eqref{eq:gamma_expression} that $\forall i,n$,
\begin{align*}
\gamma_{i,n}^{(t)} = \Bigg(&
\sum_{j=1}^M \exp\left( \dfrac{1}{2}|| y_n - \mu_i ||^2/\sigma^2 - \dfrac{1}{2}|| y_n - \mu_j ||^2/\sigma^2 \right) \\
&+ \dfrac{3200}{\sigma^2} \exp\left( \dfrac{1}{2}|| y_n - \mu_i ||^2/\sigma^2 \right)
\Bigg)^{-1}.
\end{align*}
Since $\sum_{j=1}^M \exp\left(- \dfrac{1}{2}|| y_n - \mu_j ||^2/\sigma^2 \right) \leq M \leq 3200/\sigma^2$,
\begin{align}
\dfrac{\sigma^2}{6400} \exp\left( - \dfrac{1}{2}|| y_n - \mu_i ||^2/\sigma^2 \right)
\leq \gamma_{i,n}^{(t)} \leq
\dfrac{\sigma^2}{3200} \exp\left( - \dfrac{1}{2}|| y_n - \mu_i ||^2/\sigma^2 \right). 
\end{align}
Therefore, when $\|\mu_i^{(t)} - y_n\|\leq {0.01}$, $\gamma_{i,n}^{(t)}\geq\dfrac{\sigma^2}{6400} N^{-1/2}$;
when $\|\mu_i^{(t)} - y_n\|\leq {0.015}$, $\gamma_{i,n}^{(t)}\geq\dfrac{\sigma^2}{6400} N^{-9/8}$;
when $\|\mu_i^{(t)} - y_n\|\geq {0.1}$, $\gamma_{i,n}^{(t)}\leq\dfrac{\sigma^2}{3200} N^{-50}$.

\begin{itemize}
\item
For ${n_i}\in\{1,\cdots,N-9M\}\setminus\{i_1,\cdots,i_M\}$,
\begin{align*}
\|\mu_i^{(t+1)} - y_{n_i}\| &\leq
\dfrac{ \gamma_{i,{n_i}}^{(t)} }{ \sum_{\hat{n}=1}^N \gamma_{i,\hat{n}}^{(t)} }  \|y_{n_i} - y_{n_i}\|
+ \dfrac{ \sum_{\tilde{n}\neq {n_i}} \gamma_{i,\tilde{n}}^{(t)} }{ \sum_{\hat{n}=1}^N \gamma_{i,\hat{n}}^{(t)} }  \|y_{\tilde{n}} - y_{n_i}\|
\\
&= \dfrac{ \sum_{\tilde{n}\neq {n_i}} \gamma_{i,\tilde{n}}^{(t)} }{ \sum_{\hat{n}=1}^N \gamma_{i,\hat{n}}^{(t)} }  \|y_{\tilde{n}} - y_{n_i}\|.
\end{align*}
Since $\|\mu_i^{(t)} - y_{n_i}\| \leq {0.01}$ and $\|\mu_i^{(t)} - y_{\hat{n}}\| \geq {0.1}$, $\forall\hat{n}\neq n_i$ (and that $N\geq2$),
\begin{align}
\dfrac{ \sum_{\tilde{n}\neq {n_i}} \gamma_{i,\tilde{n}}^{(t)} }{ \sum_{\hat{n}=1}^N \gamma_{i,\hat{n}}^{(t)} } \leq
\dfrac{\sum_{\hat{n} \neq {n_i}} \gamma_{i,\hat{n}}^{(t)}}{\gamma_{i,{n_i}}^{(t)}}
\leq
\dfrac{ N\cdot\dfrac{\sigma^2}{3200} N^{-50} }{ \dfrac{\sigma^2}{6400} N^{-1/2} }
\leq
10^{-10}. \label{eq:mixture_comp_ratio}
\end{align}
Hence
\begin{align*}
\|\mu_i^{(t+1)} - y_{n_i}\|
&\leq \dfrac{ \sum_{\tilde{n}\neq n_i} \gamma_{i,\tilde{n}}^{(t)} }{ \sum_{\hat{n}=1}^N \gamma_{i,\hat{n}}^{(t)} }  \|y_{\tilde{n}} - y_{n_i}\|
\\
&\leq 2\cdot10^{-10}\sup_{\hat{n}}\|y_{\hat{n}}\|
\leq 10^{-10}
\leq {0.01}.
\end{align*}

\item
Denote $\Omega_k = \{ N - 9M + 9(k-1) + 1 , \cdots, N - 9M + 9k \}\cup\{i_k\}$.
For $n_i \in \Omega_k$, $\forall k\in\{1,\cdots,M\}$,
\[
\|\mu_i^{(t+1)} - y_{n_i}\| \leq
\|\mu_i^{(t+1)} - y_{i_k}\| + \|y_{n_i} - y_{i_k}\| \leq
\|\mu_i^{(t+1)} - y_{i_k}\| + \dfrac{\sigma}{2}.
\]
And
\begin{align*}
\|\mu_i^{(t+1)} - y_{i_k}\| \leq
\left\| \sum_{\tilde{n}\in\Omega_k} \dfrac{ \gamma_{i,\tilde{n}}^{(t)} }{ \sum_{\hat{n}=1}^N \gamma_{i,\hat{n}}^{(t)} } \
\left(y_{\tilde{n}} - y_{i_k} \right) \right\|
+ \sum_{\tilde{n}\notin\Omega_k} \dfrac{\gamma_{i,\tilde{n}}^{(t)}}{ \sum_{\hat{n}=1}^N \gamma_{i,\hat{n}}^{(t)} } \|y_{\tilde{n}} - y_{i_k}\|.
\end{align*}
Define
\begin{align*}
{y_{i_k}^{avg}} =
\dfrac{\sum_{\tilde{n}\in\Omega_k} \gamma_{i,\tilde{n}}^{(t)} }{ \sum_{\hat{n}\in\Omega_k} \gamma_{i,\hat{n}}^{(t)} } y_{\tilde{n}}.
\end{align*}
Then
\begin{align*}
\|\mu_i^{(t+1)} - y_{i_k}\| \leq
\dfrac{ \sum_{\tilde{n}\in\Omega_k} \gamma_{i,\tilde{n}}^{(t)} }{ \sum_{\hat{n}=1}^N \gamma_{i,\hat{n}}^{(t)} }
\left\| {y_{i_k}^{avg}} - y_{i_k} \right\|
+ \dfrac{ \sum_{\tilde{n}\notin\Omega_k} \gamma_{i,\tilde{n}}^{(t)}}{ \sum_{\hat{n}=1}^N \gamma_{i,\hat{n}}^{(t)} } \|y_{\tilde{n}} - y_{i_k}\|.
\end{align*}

Since $\sup\limits_{\tilde{n}\in\Omega_k} \|y_{\tilde{n}} - y_{i_k}\| \leq \sigma/2$, $\|{y_{i_k}^{avg}} - y_{i_k}\|\leq \sigma/2$.
And for any $\tilde{n}\in\Omega_k$, we use induction assumption and $\sup\limits_{\tilde{n}\in\Omega_k} \|y_{\tilde{n}} - y_{i_k}\| \leq \sigma/2$ to obtain that
\[
\|\mu_i^{(t)} - y_{\tilde{n}}\| \leq
\|\mu_i^{(t)} - y_{n_i}\| + \|y_{n_i} - y_{i_k}\| + \|y_{i_k} - y_{\tilde{n}}\| \leq
0.1 + \dfrac{\sigma}{2} + \dfrac{\sigma}{2}
\leq 0.015.
\]
Hence $\gamma_{i,\tilde{n}}^{(t)}\geq\dfrac{\sigma^2}{4N} N^{-9/8}$.
Then similar to~\eqref{eq:mixture_comp_ratio},
\begin{align*}
\dfrac{ \sum_{\tilde{n}\notin \Omega_k} \gamma_{i,\tilde{n}}^{(t)} }{ \sum_{\hat{n}=1}^N \gamma_{i,\hat{n}}^{(t)} }
\leq
\dfrac{ \sum_{\tilde{n}\notin \Omega_k} \gamma_{i,\tilde{n}}^{(t)} }{ \sum_{\hat{n}\in \Omega_k} \gamma_{i,\hat{n}}^{(t)} }
\leq 10^{-10}.
\end{align*}

Therefore,
\begin{align*}
\|\mu_i^{(t+1)} - y_{n_i}\| &\leq
\dfrac{ \sum_{\tilde{n}\in\Omega_k} \gamma_{i,\tilde{n}}^{(t)} }{ \sum_{\hat{n}=1}^N \gamma_{i,\hat{n}}^{(t)} }
\left\| {y_{i_k}^{avg}} - y_{i_k} \right\|
+ \dfrac{ \sum_{\tilde{n}\notin\Omega_k} \gamma_{i,\tilde{n}}^{(t)}}{ \sum_{\hat{n}=1}^N \gamma_{i,\hat{n}}^{(t)} } \|y_{\tilde{n}} - y_{i_k}\| + \dfrac{\sigma}{2}
\\ &\leq
\|{y_{i_k}^{avg}} - y_{i_k}\| + 10^{-10} \cdot 1 + \dfrac{\sigma}{2}
\leq \sigma + 10^{-10}
\leq 0.01.
\end{align*}
\end{itemize}

It follows from induction that if $\|\mu_i^{(0)} - y_{i_k}\| \leq 0.01$, then $\|\mu_i^{(\tau)} - y_{i_k}\|\leq 0.01$, $\forall \tau>0$.
\end{proof}

\section{Detailed Experimental Settings for Gaussian Mixture Models}

We consider the same problem as that in Supplement~\ref{sec:GMM_proof} of inferring mean parameters $\pmu=(\mu_1,\cdots,\mu_M)\in\mathbb{R}^{d\times M}$ in a Gaussian mixture model with $M$ mixtures from $N$ data points $\vy=(y_1,\cdots,y_N)$:
\begin{align}
&p\left(y_n | \pmu\right) = \sum_{i=1}^M \dfrac{\lambda_i}{Z_i} \exp\left( - \dfrac{1}{2} (y_n - \mu_i)^\rT \Sigma_i^{-1} (y_n - \mu_i) \right)
+ \left( 1 - \sum_{i=1}^M \lambda_i \right) p_0(y_n),
\label{eq:GMM}
\end{align}
where
the covariances $\Sigma_i$ are isotropic and uniform across all mixture components: $\Sigma_i = \Sigma = \sigma^2 \mI$.
The constant mixture $p_0(y_n) = \ind{\|y_n\|\leq R} / Z_0$ represents crude observations of the data, which are distributed in a bounded region: $\lrn{y_n}\leq R$.
The objective function is given by the log posterior distribution: $U(\pmu) = -\log p(\pmu) - \sum_{n=1}^N \log p\left(y_n | \pmu\right)$,
where we take the prior to be
\begin{align}
p(\pmu) \propto \exp\left( - m \left(\|\pmu\|_F - {\sqrt{M}} R\right)^2 \ind{\|\pmu\|_F \geq {\sqrt{M}} R} \right).
\label{eq:GMM_prior}
\end{align}
\begin{figure}
\centering
\includegraphics[scale=0.35]{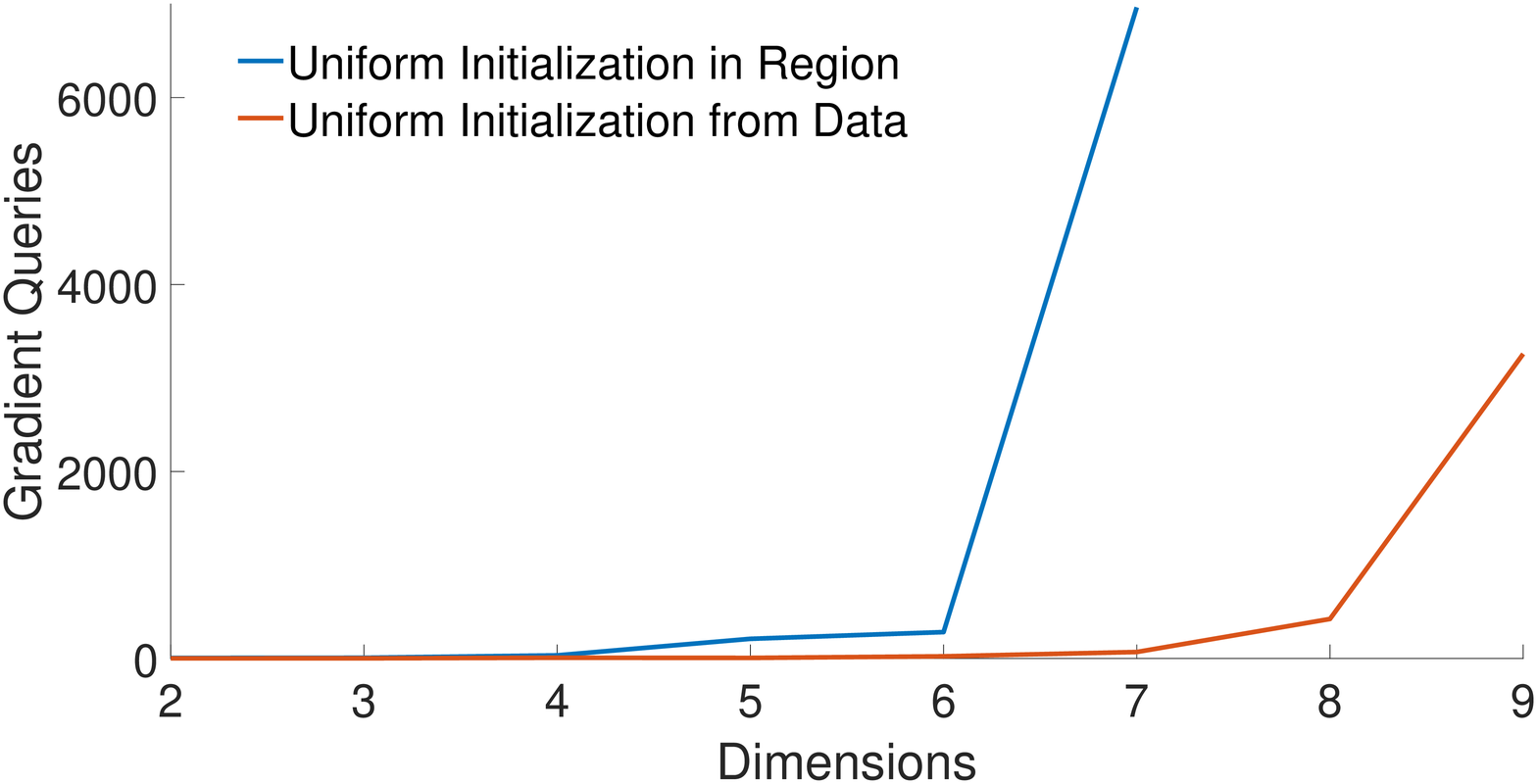}
\caption{Experimental results: scaling of the number of gradient queries required for EM with random initialization uniformly in the ball of radius $R$ and uniformly from the data.}
\label{fig:GMM_experiment2}
\end{figure}

Similar to the setting in Supplement~\ref{sec:dataset}, we take $\dfrac{\lambda_i}{Z_i}$ to be $\sigma^2/1000$, where the variance $\sigma = 1/\sqrt{d}$, so that the mixtures are well separated from each other.

We consider a synthetic dataset, $\{y_1,\cdots,y_N\}$, with sparse entries: only $\lfloor \log_2 d \rfloor$ of the entries in each data point $y_n$ are nonzero.
Indices of the nonzero entries are uniformly distributed over the set $\{1,\cdots,d\}$.
All the nonzero entries follow a uniform distribution on $[-1,1]$.
Also assume that the number of mixtures $M = \lfloor \log_2 d \rfloor$.
Hence the radius containing the data $R=2 \sqrt{M \lfloor \log_2 d \rfloor} = 2 \lfloor \log_2 d \rfloor$.
We generate $N=2^d$ data points following this rule.

We let the dimension $d$ range from $2$ to $32$ and recorded the number of gradient entries required for EM (with random initialization from the data) and ULA to converge.
The results were averaged over $20$ trials of experiments.
When dimension $d\geq10$, too many gradient queries are required for EM to converge, so that an accurate estimate of convergence time is not available.

For EM, we measured its accuracy in terms of the objective function value $U$ and require $U(\pmu_K) - U(\widehat{\pmu^*}) < 10^{-6}$ to conclude that $\pmu_K$ has converged close enough to $\pmu^*$.
For ULA, we measured its accuracy in terms of both the expected objective function value $\E{U(\pmu)}$ (or equivalently the cross entropy between the sampled distribution and the posterior) and the expected mean parameters $\E{\pmu}$.
We required both $\left| \frac{1}{K}\sum_{k=1}^K U(\pmu_k) - \widehat{\Ep{p^*}{U(\pmu)}} \right| < 10^{-6}$ and $ \left\| \frac{1}{K}\sum_{k=1}^K \pmu_k - \widehat{\Ep{p^*}{\pmu}} \right\|_F < 10^{-3}$ (which are of comparable scales) to assess the convergence of the sampling algorithm.

To estimate the reference value $\widehat{\pmu^*}\in\mathbb{R}^{d\times M}$, we run EM $1000$ times longer than the number of required steps found for the previous experiment with dimension $d-1$.
If estimates from $20$ different initializations differed by less than $10^{-8}$, we  accepted $\widehat{\pmu^*}$.
Otherwise, we increased the number of steps by $10$ times.
We similarly estimated $\widehat{\Ep{p^*}{U(\pmu)}}$ and $\widehat{\Ep{p^*}{\pmu}}$ by long runs of ULA (also $1000$ times longer than the number of required steps found for dimension $d-1$).
If estimates from $20$ different initializations differed by less than $10^{-8}$ for $\widehat{\Ep{p^*}{U(\pmu)}}$ and $10^{-5}$ for $\widehat{\Ep{p^*}{\pmu}}$, we accepted the estimates.
Otherwise, we increased the number of steps by $10$ times.

We also compared EM with random initialization uniformly in the ball of radius $R$ against that with uniform initialization from the data points.
We observed in Fig.~\ref{fig:GMM_experiment2} that initializing uniformly in the ball of radius $R$ leads to poorer convergence, implying that there are more local minima of $U$ than merely those nearby the data.

\end{appendices}

\bibliographystyle{plain}
\bibliography{CompleteSampling}

\end{document}